\documentclass{article}
\usepackage{iclr2027_conference,times}
\usepackage{bm}
\usepackage{xcolor}
\usepackage{amsmath,amssymb,amsthm,mathtools,bm}
\usepackage{booktabs,multirow,array,tabularx}
\usepackage{graphicx,subcaption}\usepackage{wrapfig}
\usepackage[table,dvipsnames]{xcolor}
\usepackage{microtype}
\usepackage{enumitem}
\usepackage{etoolbox}
\usepackage{needspace}\usepackage{placeins}
\usepackage{algorithm}
\usepackage[noend]{algpseudocode}
\usepackage{siunitx}
\usepackage{url}
\colorlet{LinkColor}{MidnightBlue}
\definecolor{TableAccent}{RGB}{55,91,130}
\definecolor{OurMethodColor}{HTML}{119F87}
\usepackage{bm}
\newcommand{\bestvalue}[1]{\bm{#1}}
\newcommand{\ourmethod}[1]{\textcolor{OurMethodColor}{\textbf{#1}}}
\newcommand{\resultstablestyle}{%
  \small
  \arrayrulecolor{black!55}%
  \setlength{\heavyrulewidth}{0.6pt}%
  \setlength{\lightrulewidth}{0.35pt}%
  \setlength{\cmidrulewidth}{0.3pt}%
  \renewcommand{\arraystretch}{1.2}%
}
\colorlet{UncertaintyColor}{gray}
\usepackage[colorlinks=true,allcolors=LinkColor]{hyperref}

\setcitestyle{authoryear,round,citesep={;},aysep={,},yysep={;}}
\graphicspath{{figures/}}

\patchcmd{\paragraph}{1.5ex plus 0.5ex minus .2ex}{0pt}{}{
  \PackageError{ftfc}{Could not compact paragraph heading spacing}{}
}
\patchcmd{\subparagraph}{1.5ex plus 0.5ex minus .2ex}{0pt}{}{
  \PackageError{ftfc}{Could not compact subparagraph heading spacing}{}
}

\newif\ificlrhascitations
\makeatletter
\pretocmd{\NAT@citex}{\global\iclrhascitationstrue}{}{
  \PackageWarning{ftfc}{Could not track natbib citations}
}
\pretocmd{\nocite}{\global\iclrhascitationstrue}{}{
  \PackageWarning{ftfc}{Could not track nocite commands}
}
\makeatother

\newcommand{\R}{\mathbb{R}}
\newcommand{\E}{\mathbb{E}}

\newcommand{\DKL}{\mathrm D_{\mathrm{KL}}}
\DeclareMathOperator{\argmin}{arg\,min}

\newtheorem{theorem}{Theorem}
\newtheorem{proposition}[theorem]{Proposition}

\newtheorem{corollary}[theorem]{Corollary}

\theoremstyle{definition}

\definecolor{TakeawayBorder}{RGB}{55,91,130}
\definecolor{TakeawayBackground}{RGB}{241,246,251}

\newcommand{\uncertainty}[1]{\,{\color{UncertaintyColor}\scriptstyle\ensuremath{\pm#1}}}

\newcommand{\methodcite}[1]{{\scriptsize\citep{#1}}}

\title{Fenchel Tilting: \\Weighted Correction for Efficient Finetuning of Generative Models}
\author{  
Maksim Bobrin\textsuperscript{1,2}\thanks{Corresponding author: \texttt{maxs.bobrin@gmail.com}} \quad  
Maksim Zhdanov\textsuperscript{1,3} \quad  
Dmitry Dylov\textsuperscript{1,2}\\[0.35em]  
\textsuperscript{1}AXXX \quad \textsuperscript{2}Applied AI Institute \quad \textsuperscript{3}MBZUAI}
\iclrfinalcopy
\begin{document}
\maketitle
\thispagestyle{fancy}
\begin{abstract}
Adapting a pretrained generative model to an arbitrary preference expressed 
as a utility function underlies reward alignment, guided design, 
and constraint satisfaction, enabling diverse applications.
Existing fine-tuning methods trade off generality against computational cost: they either
restrict the family class of supported preferences to keep optimization simple 
or preserve generality at the expense of efficiency.
We introduce Fenchel Tilt Flow Control (FTFC), which 
decouples utility optimization from generative-model fitting.
FTFC first optimizes for a target distribution by jointly fitting an effective reward and density-ratio weights on pretrained samples. Method combines the utility’s variational structure with Fenchel duality, supporting general 
$f$-divergence penalties that determine how rewards are transformed into an distribution-correction weights. These weights are then frozen and used to modify a diffusion or flow model in a single stage of importance-weighted denoising or flow matching, without differentiating through sampling trajectories.
We establish exact duality for concave utilities under suitable conditions and show that weighted fitting reproduces the optimal target distribution for a given utility.
Across image and molecule generation benchmarks, FTFC improves over baselines on diverse preference functions, while also being up to $20\times$ more efficient.
Proposed method enables adaptation beyond expected-reward maximization without complex optimization, while preserving robustness for more general class of the utility functions compared to baselines.
\end{abstract}

\section{Introduction}
\label{sec:introduction}
Diffusion and flow models support text-to-image synthesis \citep{rombach2022latent,esser2024scaling}, molecular generation and drug design \citep{dunn2024flowmol,schneuing2024diffsbdd}, and robotic control \citep{chi2023diffusionpolicy}. Reproducing the training distribution, however, does not directly optimize qualities such as aesthetic appeal, prompt fidelity, or molecular validity. Reward-based fine-tuning adapts a pretrained model to such external scores \citep{xu2023imagereward}, but average reward cannot express all preferences: molecular discovery may prioritize rare high-quality candidates, whereas reliable generation may emphasize the worst outcomes. This motivates optimizing a utility of the full output distribution while remaining close to the pretrained model \citep{de2025flow,wang2026efficient}.

For nonlinear utilities, the functional gradient depends on the generated distribution itself \citep{de2025flow}. Flow Density Control (FDC) handles this dependence by alternating utility linearization and generative fine-tuning, using solvers such as Adjoint Matching \citep{domingo2025adjoint}; it therefore requires an outer mirror-descent loop and an inner control solver. Tail-Aware Flow Fine-Tuning (TFFT) separates risk-sensitive threshold selection from a single entropy-regularized update \citep{wang2026efficient}, but its decomposition is specific to CVaR under KL regularization. Broader structured utilities and divergences require the effective pseudo-reward and induced target distribution to be determined jointly.

Motivated by distribution correction estimation (DICE) in off-policy reinforcement learning \citep{nachum2019dualdice,lee2021optidice}, we introduce \emph{Fenchel Tilt Flow Control} (FTFC). FTFC first selects a target distribution relative to the pretrained model and then fits the generator. For concave utilities with suitable $f$-divergence penalties, Fenchel duality yields a supporting reward and closed-form density-ratio weights that specify the probability-mass shift. The frozen weights then reweight one stage of native denoising or flow-matching training \citep{zhang2025energyweighted,potaptchik2025tilt}. Precomputed endpoints and rewards, combined with fresh native noise, avoid reward gradients and differentiation through sampling trajectories. FTFC thus separates \emph{selecting the endpoint distribution} from \emph{learning the generative map}.

\Needspace{4\baselineskip}
Our contributions are:
\begin{itemize}[leftmargin=*,topsep=2pt,itemsep=2pt,parsep=0pt]
\item \textbf{Dual formalization of general utility.} We cast utility optimization as normalized density-ratio fitting with supported $f$-divergence penalties, including KL and Pearson $\chi^2$. For concave utilities, the Fenchel dual jointly recovers a marginal reward and optimal weights; we also give necessary primal--dual gap conditions.
\item \textbf{Simple and scalable fine-tuning.}
FTFC reduces utility-based adaptation to two simple steps: optimize the
target weights on cached samples, then freeze them and fine-tune the
generator by reweighting its native denoising or flow-matching loss.
No reward gradients, trajectory differentiation, or specialized
generative optimization is required.
\item \textbf{Efficient adaptation.}
Across molecular and image benchmarks, FTFC is $10$--$20\times$ faster
than prior fine-tuning methods. By separating target selection from
generator fitting, it avoids repeated reward optimization and
differentiation through generation trajectories.
\end{itemize}
\section{Related Work}
\label{sec:related-work}
\paragraph{General Utility Optimization.}
Flow Density Control (FDC) optimizes nonlinear utilities by alternating utility linearization with entropy-regularized fine-tuning \citep{de2025flow}, implemented through Adjoint Matching \citep{domingo2025adjoint}. Tail-aware Flow Fine-Tuning (TFFT) instead selects a CVaR threshold before one KL-regularized fine-tuning stage \citep{wang2026efficient}. FTFC extends this separation to structured utilities and supported $f$-divergences: it determines the marginal reward and target weights before updating the generator. Under KL-regularized CVaR, it recovers the TFFT target objective.

\paragraph{Distribution Correction Estimation (DICE).}
DICE methods address off-policy distribution shift by estimating state-action occupancy ratios rather than trajectory importance weights. DualDICE estimates these ratios for evaluation \citep{nachum2019dualdice}, whereas OptiDICE uses Fenchel duality to optimize a divergence-regularized occupancy measure from offline data \citep{lee2021optidice}. FTFC transfers this view to generative fine-tuning by estimating target-to-pretrained density ratios. For nonlinear utilities, the effective reward depends on the unknown target distribution; Fenchel duality resolves this coupling by jointly recovering the marginal reward and normalized weights, with endpoint normalization replacing RL occupancy-flow constraints.

\paragraph{Learning the target generator.}
A specified target distribution can be learned through stochastic control, including non-KL reward tilts \citep{tang2024finetuning}, or weighted flow and diffusion regression \citep{zhang2025energyweighted}. DiffCon derives general-$f$ weighted denoising from a trajectory-distribution objective and an equivalent KL transition-control form \citep{yang2026diffusioncontroller}. Tilt Matching avoids reward gradients and trajectory backpropagation \citep{potaptchik2025tilt}, while online reward-weighted flow matching combines regression, fresh samples, and Wasserstein regularization \citep{fan2025online}. Guidance offers another route \citep{feng2025guidance}. Diffusion-DICE learns it from occupancy ratios \citep{mao2024diffusiondice} and FTFC uses its in-sample estimator only for optional guidance. Its main algorithm instead fits the target distribution with the native denoising or flow-matching loss while keeping utility-dependent weights fixed.
\section{Problem Setup}
\label{sec:problem-setup}
We adopt the notation of Flow Density Control (FDC)
\citep{de2025flow} and Tail-aware Flow Fine-Tuning (TFFT)
\citep{wang2026efficient}. We will formalize the problem setting in the deterministic flow-based framework.

\paragraph{Generative models as policies.}
Let $\mathcal{P}(\mathcal{X})$ denote the Borel probability measures on the
sample space $\mathcal X=\R^d$. A flow
model is specified by a parametrized velocity
field $u_\theta(x,t)$: given the current sample $x$ and generation time $t$,
it predicts the direction and speed of the sample's motion, i.e \textit{vector field}. Pretraining learns
$\theta$ so that transporting initial noise $X_0\sim p_0$ produces samples
approximating the data distribution $p_{\mathrm{data}}$
\citep{lipman2023flow,lipman2024flowmatchingguidecode}.
We interpret $(x,t)$ as a state and the predicted velocity as an
action. The generative network itself therefore defines a deterministic
\emph{policy} $\pi_\theta(x,t):=u_\theta(x,t)$, abbreviated as $\pi$.
Sampling follows $X_t=\psi_t^\pi(X_0)$ with
\begin{equation}
  \frac{d}{dt}\psi_t^\pi(x)
  =\pi\!\left(\psi_t^\pi(x),t\right),
  \qquad \psi_0^\pi(x)=x,
  \qquad t\in[0,1].
  \label{eq:setup-flow}
\end{equation}
We denote the distribution of $X_t$ by $p_t^\pi$, with $p_1 ^\pi$ be an endpoint (final) distribution. The pretrained weights
$\theta_{\mathrm{pre}}$ define the fixed reference output law
$q:=p_1^{\mathrm{pre}}$.
Thus $p_0$ describes input noise, while $q$ describes generated samples.

\paragraph{Generative optimization.}
Fine-tuning updates the network pretrained parameters $\theta$, changing its movement
rule $\pi_\theta$ and hence its output distribution $p_1^{\pi_\theta}$.
Given a general (possibly, non-linear) utility function $\mathcal{F}:\mathcal{P}(\mathcal{X})\to\R$, a divergence
$\mathcal{D}$ measuring departure from $q$, and $\alpha\geq0$, the objective is
\begin{equation}
  \begin{gathered}[b]
    \text{\small\bfseries\color{TableAccent}General Utility}\\[-2pt]
    \boxed{\sup_{\pi\in\Pi}\;\mathcal{G}(p_1^\pi),
    \qquad
    \mathcal{G}(p):=\mathcal{F}(p)-\alpha\mathcal{D}(p\|q).}
  \end{gathered}
  \label{eq:go-objective}
\end{equation}
For a flow backbone, $\Pi=\{\pi_\theta:\theta\in\Theta\}$ is the family of velocity fields
represented by the model, with allowed parameter set $\Theta$ and fixed
source $p_0$.
For a reward $r:\mathcal{X}\to\R$, choosing
$\mathcal{F}(p)=\E_{X\sim p}[r(X)]$ and $\mathcal{D}=\DKL$ recovers
standard reward fine-tuning (mean maximizing). TFFT \citep{wang2026efficient} instead uses left- or right-CVaR of
$r(X)$, with quantile level $\beta\in(0,1)$, to optimize the lower or upper
reward tail by reformulating CVaR objective with an equivalent variational form.

\paragraph{Endpoint distribution correction.}
The distribution formulation extends to an $f$-\textit{divergence},
$\mathcal D=D_f$. If $p\ll q$, its density ratio
$w=dp/dq$ gives $p=wq$, meaning $p(A)=\int_A w(x)q(dx)$, and
$D_f(p\|q)=\E_q[f(w)]$ for a convex generator $f$ with $f(1)=0$.
KL is the particular choice $f(s)=s\log s$, with $0\log0=0$:
\[
  \DKL(p\|q)=\E_p[\log w]=\E_q[w\log w].
\]
For this formulation, restrict $\Pi$ to policies with $p_1^\pi\ll q$
and finite objective terms. Write $w_\pi:=dp_1^\pi/dq$ and define
$\mathcal{W}_\Pi:=\{w_\pi:\pi\in\Pi\}$, the set of ratios actually
produced by these policies. The objective depends on $\pi$ only through
its terminal law: $p_1^\pi=w_\pi q$ implies
$\mathcal{G}(p_1^\pi)=\mathcal{F}(w_\pi q)-\alpha\E_q[f(w_\pi(x))]$.
Conversely, every $w\in\mathcal{W}_\Pi$ comes from at least one admissible
policy, by definition. Both parameterizations therefore attain exactly the
same objective values. Taking suprema yields
\begin{equation}
  \begin{aligned}
    \sup_{\pi\in\Pi}\mathcal{G}(p_1^\pi)
    &=\sup_{w\in\mathcal{W}_\Pi}
      \left\{\mathcal{F}(wq)-\alpha\E_q[f(w)]\right\}\\
    &\leq\sup_{w\geq0,\;\E_q[w]=1}
      \left\{\mathcal{F}(wq)-\alpha\E_q[f(w)]\right\}.
  \end{aligned}
  \label{eq:setup-endpoint-objective}
\end{equation}
Every $w_\pi$ is nonnegative and satisfies $\E_q[w_\pi]=1$. The inequality
enlarges $\mathcal{W}_\Pi$ to all ratios satisfying these constraints, with
finite objective terms; absolute continuity alone does not guarantee that
every such ratio is realizable by the model.
An optimizer $w^\star$ of this relaxed problem, when it exists, defines
the target $p^\star=w^\star q$, which provides a connection between optimal distribution maximizing reward and pretrained model.

\begin{figure}[!ht]
  \centering
  \includegraphics[width=0.92\linewidth]{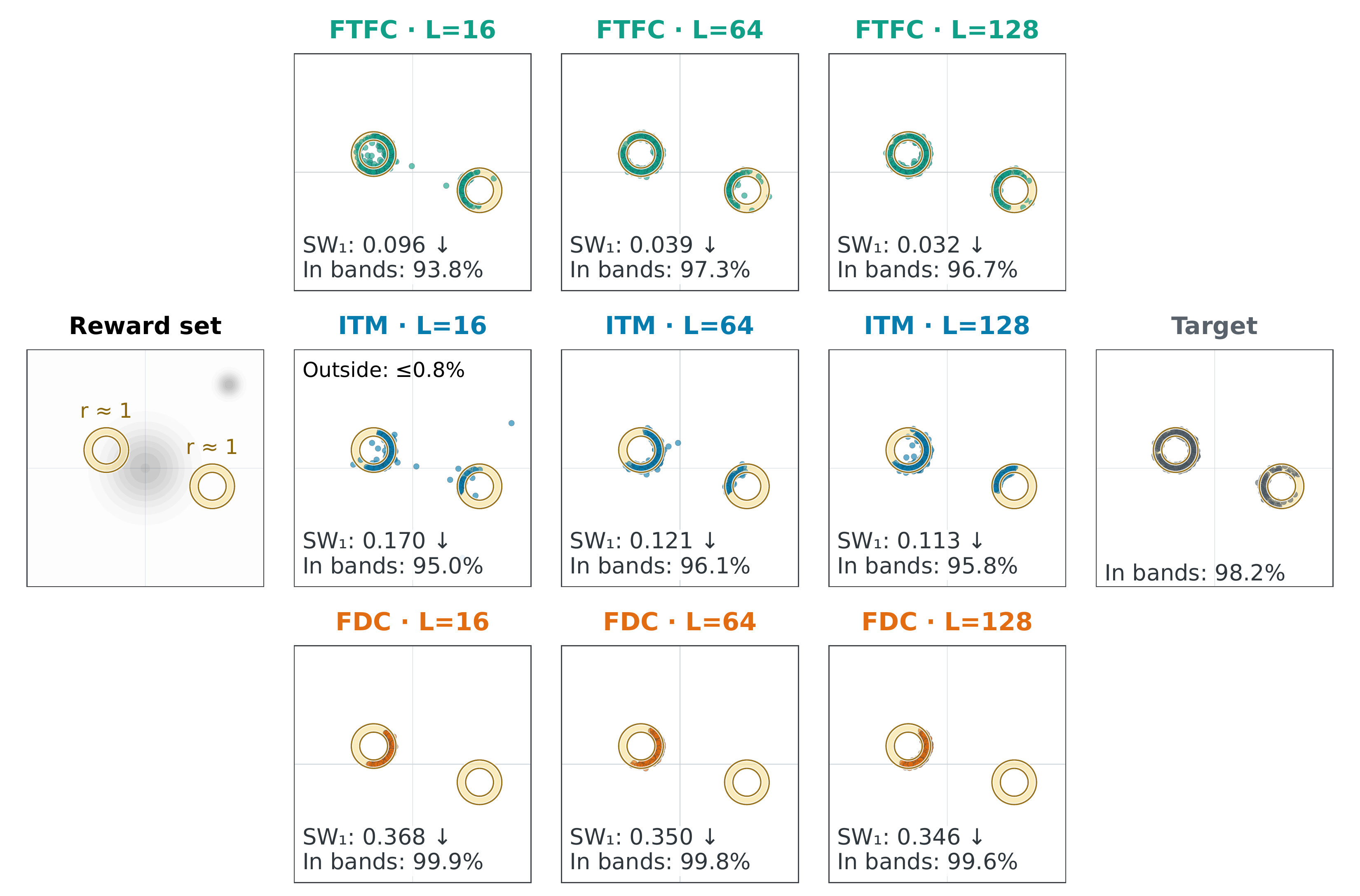}
  \caption{\textbf{Left-tail CVaR on two rings.}
  Gold bands correspond to $d(x)\leq0.15$. $L$ counts number of sampling steps.
  SW$_1$ is distance to the independent numerical target (right, ground truth). FTFC converges faster to target according to the utility function. Comparison is made with FDC \citep{de2025flow} and ITM \citep{potaptchik2025tilt}. See Section \ref{sec:motivation} for details.}
  \label{fig:toy2d-left-cvar}
\end{figure}

\section{Method: Fenchel Tilt Flow Control}
\label{sec:method}\label{ftfc:sec:method}
FTFC finds the target distribution \textit{before} fitting the generative map through dual optimization
(Figures~\ref{fig:ftfc-method}) by finding an near-optimal weight correction \ref{thm:ftfc-general-dual}.
We consider the relaxed endpoint problem in
Eq.\eqref{eq:setup-endpoint-objective}, with $\alpha>0$.

\subsection{Motivation}
\label{sec:motivation}

Consider two rings with reward $r(x)=1-d(x)^2$, where $d(x)$ is the
distance to the nearer ring. Since points near either ring receive similar
reward, high reward does not require covering both rings or their full
support. We optimize the lowest-$20\%$ CVaR with a KL penalty relative to
the pretrained distribution $q$, so the objective depends not only on
sample quality but also on how probability mass is allocated.

Figure~\ref{fig:toy2d-left-cvar} compares FTFC with Implicit Tilt Matching
(ITM; \citealp{potaptchik2025tilt}) and FDC (\citealp{de2025flow}), using
the same calibrated CVaR reward. FDC attains higher raw CVaR but misses the
right ring, while ITM covers both rings but concentrates on particular arcs.
FTFC instead most closely matches the numerical target in sliced-Wasserstein
distance. The example exposes the key motivation for FTFC: knowing which
samples are rewarding is not enough to determine how probability should be
distributed among them. FTFC resolves this ambiguity first, by selecting the
target distribution by cheap dual optimization, and only then fits the generator to that fixed target
using weighted denoising or flow matching~\eqref{eq:ftfc-native-population}.
Additional objectives are given in Appendix~\ref{app:toy2d-entropy}.

\subsection{Optimal correction in the closed form}
\label{ftfc:sec:setup}
When $\mathcal F$ is differentiable, let $g_p:\mathcal X\to\mathbb R$
denote its \emph{functional derivative at $p$} (first
variation). On a finite sample space, this is
$g_p(x_i)=\partial\mathcal F(p)/\partial p_i$.
The subscript $p$ identifies the distribution being perturbed, and
$g_p(x)$ is the derivative at point $x$.
Intuitively, it measures the gain from redistributing probability: moving a small mass from $y$
to $x$ changes utility at rate $g_p(x)-g_p(y)$.
To derive the optimal weight, perturb $w$ to $w+\varepsilon v$ in
\eqref{eq:setup-endpoint-objective}, keeping $\E_q[v]=0$ so the law
stays normalized. Since the distribution changes by $\varepsilon vq$,
the first variation is
\[
 \left.\frac{d}{d\varepsilon}\mathcal G((w+\varepsilon v)q)
 \right|_{\varepsilon=0}
 =\E_q\!\left[(g_{wq}-\alpha f'(w))v\right].
\]
At an interior optimum this vanishes for every admissible such $v$.
Hence the net marginal utility $g_{p^\star}-\alpha f'(w^\star)$ must
be a constant $\nu^\star$, $q$-almost everywhere; otherwise moving mass
from a lower value to a higher one would improve the objective. Thus
\begin{equation}
 \alpha f'(w^\star(x))=g_{p^\star}(x)-\nu^\star,
        \qquad p^\star=w^\star q.
 \label{eq:ftfc-weight-stationarity}
\end{equation}
Here $\nu^\star$ enforces normalization. For KL,
$f(w)=w\log w-w+1$ has derivative $f'(w)=\log w$.
Thus~\eqref{eq:ftfc-weight-stationarity} gives
$w^\star(x)=\exp((g_{p^\star}(x)-\nu^\star)/\alpha)$;
choosing $\nu^\star$ so that $\E_q[w^\star]=1$ yields
\begin{equation}
 \begin{gathered}[b]
 \text{\small\bfseries\color{TableAccent}KL Weights}\\[-2pt]
 \boxed{w^\star(x)=
   \frac{\exp(g_{p^\star}(x)/\alpha)}
        {\E_{X\sim q}[\exp(g_{p^\star}(X)/\alpha)]}.}
 \end{gathered}
 \label{eq:ftfc-marginal-gibbs}
\end{equation}
The denominator normalizes the weights so that
$\E_q[w^\star]=1$.
Table~\ref{tab:ftfc-divergence-responses} in
Appendix~\ref{app:ftfc-new-proofs} lists $f$, $f'$, together with the closed form
weight formulas for other divergences.
For expected reward, $\mathcal F(p)=\E_p[r]$ and $g_p(x)=r(x)$.
For nonlinear utility, $g_p$ plays the role of a \emph{marginal reward}:
moving probability toward larger values of $g_p$ improves utility to
first order, but these values depend on the current distribution.
Equation~\eqref{eq:ftfc-marginal-gibbs} therefore requires the marginal
reward at the unknown optimum $p^\star$. We next use dual optimization
to determine this reward and the normalized density ratios together.

\subsection{Fenchel duality and weight recovery}
\label{ftfc:sec:compiler}
To find the unknown marginal reward $g_{p^\star}$ from
Section~\ref{ftfc:sec:setup}, we now optimize over candidate reward
functions $g:\mathcal X\to\mathbb R$.
Here $\mathcal F$ is the \emph{same utility functional} as in
\eqref{eq:go-objective}. We assume it is proper,
closed, and concave. For each $g$, define $C(g)$ as the smallest
scalar offset satisfying
$\mathcal F(p)\leq C(g)+\E_p[g]$ for every normalized
$q$-supported law $p$. Equivalently,
\begin{equation}
 C(g):=\sup_{p'}\{\mathcal F(p')-\E_{p'}g\},\qquad
 \mathcal F(p)=\inf_g\{C(g)+\E_pg\}.
 \label{eq:ftfc-envelope}
\end{equation}
At a differentiable $p$, concavity makes the bound tight for
$g=g_p$. More generally, taking the tightest of these bounds recovers
$\mathcal F$ exactly by concave biconjugacy
\citep{rockafellar1970convex}. This representation alone does not find
$g_{p^\star}$: we must optimize $g$ and recover its normalized weights.

For any feasible $w\geq0$, $\E_qw=1$, candidate reward $g$, and scalar
label{eq:ftfc$\nu$ with finite terms, the regularized objective satisfies
\[
 \begin{aligned}
 \mathcal G(wq)
 &\leq C(g)+\E_q[wg-\alpha f(w)]\\
 &=C(g)+\nu+\E_q[(g-\nu)w-\alpha f(w)]\\
 &\leq\underbrace{C(g)+\nu+
   \alpha\E_q f_+^*((g-\nu)/\alpha)}_{D_f(g,\nu)}.
 \end{aligned}
\]
Appendix~\ref{app:ftfc-reward-bound} derives the full chain:
the first inequality uses the utility bound in \eqref{eq:ftfc-envelope}.
The equality uses $\nu(1-\E_qw)=0$, introducing the normalization multiplier.
The last inequality maximizes each nonnegative weight separately,
using $f_+^*(u):=\sup_{v\geq0}\{uv-f(v)\}$.
Minimizing this bound recovers the optimal value of the endpoint
relaxation in \eqref{eq:setup-endpoint-objective}, under the
strong-duality conditions in Theorem~\ref{thm:ftfc-general-dual}.
The weights attaining the last inequality are recovered by conjugacy,
with $\nu$ chosen to normalize them:
\begin{equation}
 \begin{gathered}[b]
 \text{\small\bfseries\color{TableAccent}Fenchel Dual}\\[-2pt]
 \boxed{\begin{aligned}
 D_f(g,\nu)&:=C(g)+\nu+\alpha\E_q f_+^*((g-\nu)/\alpha),\\
 (g^\star,\nu^\star)&\in\arg\min_{g,\nu}D_f(g,\nu),\\
 w^\star(x)&\in\partial f_+^*((g^\star(x)-\nu^\star)/\alpha),
          \qquad \E_qw^\star=1.
 \end{aligned}}
 \end{gathered}
 \label{eq:ftfc-functional-fdual}
\end{equation}
The subgradient allows zero weights. The multiplier $\nu$ enforces
normalization within the weight response; rescaling the weights
afterward is generally invalid beyond KL.
For example, half-Pearson uses $f(w)=(w-1)^2/2$ and
$w=(1+(g-\nu)/\alpha)_+$. This conjugate recovery is the DICE connection
\citep{lee2021optidice}, with endpoint normalization replacing
occupancy-flow constraints.
\begin{theorem}[Duality and optimality gap]
\label{thm:ftfc-general-dual}
For any feasible $p=wq$ and finite primal and dual terms, put
$u=(g-\nu)/\alpha$. Then
\begin{equation}
 \begin{aligned}
 D_f(g,\nu)-\mathcal G(p)
 &=\underbrace{C(g)+\E_pg-\mathcal F(p)}_{\text{supporting-reward slack}}\\[-2pt]
 &\quad+\underbrace{\alpha\E_q[f(w)+f_+^*(u)-uw]}_{\text{weight-response residual}}
 \ \geq0.
 \end{aligned}
 \label{eq:ftfc-certificate-main}
\end{equation}
Thus the gap bounds any improvement in $\mathcal G$ achievable by
another feasible endpoint law.
Under strong duality and attainment, both residuals vanish at an
optimum and \eqref{eq:ftfc-functional-fdual} recovers its weights.
A sufficient setting is a finite bank with positive reference masses,
finite continuous concave utility, and $f$ finite continuous and
strictly convex on $[0,\infty)$, differentiable on $(0,\infty)$,
and superlinear.
\end{theorem}
At zero gap, $C(g^\star)+\E_{p^\star}g^\star=\mathcal F(p^\star)$:
the affine upper bound in \eqref{eq:ftfc-envelope} touches $\mathcal F$
at the recovered law $p^\star=w^\star q$. At a differentiable interior
optimum, their first variations agree, identifying $g^\star$ with
$g_{p^\star}$ up to an additive constant. Thus
\eqref{eq:ftfc-functional-fdual} determines the reward and law together.
Appendix~\ref{app:ftfc-new-proofs} proves the result.

\subsection{Target Optimization for Structured Utilities}
\label{ftfc:sec:practical-calibration}

The functional dual in~\eqref{eq:ftfc-functional-fdual} optimizes over an
arbitrary reward function $g$, which is generally impractical. For many
structured utilities, however, the unknown effective reward has only a
small number of parameters. We assume and utilize this structure to optimize the
target weights on the cached sample bank before fitting the generator.

\paragraph{Feature parametrized utilities.}
Consider
$\mathcal F(p)=\E_p[r]-\Psi(m_p)$ with $m_p=\E_p[\phi]$
where $\phi:\mathcal X\to\mathbb R^k$ contains the statistics relevant to
the utility and $\Psi$ is proper, closed, and convex. Its marginal reward is
\[
    g_p(x)=r(x)-\nabla\Psi(m_p)^\top\phi(x),
\]
so instead of optimizing an unrestricted $g$, it suffices to determine the
$k$ marginal costs $z=\nabla\Psi(m_p)$. Using
$\Psi^*(z)=\sup_m\{z^\top m-\Psi(m)\}$ gives
\begin{equation}
 \begin{gathered}[b]
 \text{\small\bfseries\color{TableAccent}Feature Dual}\\[-2pt]
 \boxed{\begin{aligned}
 \mathcal F(p)&=\inf_z\{\Psi^*(z)+\E_p[r-z^\top\phi]\},\\
 (z^\star,\nu^\star)&\in\arg\min_{z,\nu}
 \{\Psi^*(z)+\nu+\alpha\E_q f_+^*((r-z^\top\phi-\nu)/\alpha)\},\\
 g^\star&=r-z^{\star\top}\phi,\qquad
 z^\star\in\partial\Psi(\E_{p^\star}\phi).
 \end{aligned}}
 \end{gathered}
 \label{eq:ftfc-feature-dual}
\end{equation}
Thus target optimization reduces to solving for $(z^\star,\nu^\star)$
on cached samples. The resulting $g^\star$ is converted into normalized
target weights by~\eqref{eq:ftfc-functional-fdual}. Representative utilities
and their corresponding $(\phi,\Psi)$ are given in
Appendix~\ref{app:ftfc-utility-catalog}; numerical details are given in
Appendix~\ref{app:ftfc-solve-statistics}.

\paragraph{CVaR utilities.}
CVaR does not require moment features: its effective reward is determined
by a scalar threshold. Upper-CVaR follows the concave formulation above
with scalar threshold optimization. Lower-CVaR is convex in the output law
and therefore requires a separate argument. On a finite sample bank, we
show that a globally optimal lower-CVaR threshold can be chosen among the
observed rewards (Proposition~\ref{prop:ftfc-general-tail-knots}).
Searching these thresholds yields the target weights, which are then frozen
for the same generator-fitting stage. Benchmark-specific pseudo-rewards and
threshold objectives are given in Appendix~\ref{app:ftfc-tail-benchmarks}.

\begin{figure}[tb]
\centering
\begingroup
\setlength{\fboxsep}{9pt}
\setlength{\fboxrule}{0.5pt}
\fbox{\begin{minipage}{\dimexpr\linewidth-2\fboxsep-2\fboxrule\relax}
\normalsize
\noindent\makebox[\linewidth][s]{%
$\mathcal F$\hfill $\xrightarrow[\text{\tiny Eq.~$\eqref{eq:ftfc-functional-fdual}$}]{\text{\small$\min_{g,\nu}D_f(g,\nu)$}}$\hfill $(g^\star,\nu^\star)$\hfill $\xrightarrow[\text{\tiny Eq.~$\eqref{eq:ftfc-functional-fdual}$}]{\text{\small$w^\star\in\partial f_+^*((g^\star-\nu^\star)/\alpha)$}}$\hfill $p^\star=w^\star q$\hfill $\xrightarrow[\text{\tiny Eq.~$\eqref{eq:ftfc-native-population}$}]{\text{\small$\min_\theta\mathcal L(\theta;w^\star)$}}$\hfill $\widehat\pi$%
}
\end{minipage}}
\endgroup
\caption{\textbf{FTFC: select a target, then fit it.} The dual determines $(g^\star,\nu^\star)$; the DICE link recovers normalized endpoint weights, which are frozen for native fitting.}\label{fig:ftfc-method}
\end{figure}

\subsection{Weighted denoising and flow matching}
\label{ftfc:sec:realization}
Calibration in Section~\ref{ftfc:sec:practical-calibration} returns
weights $w$ defining the target $p=wq$. Freeze these weights and
initialize the generator from its pretrained parameters.
The calibrated reward affects the generator through the weighted
training loss. For each endpoint $X\sim q$, independently sample a
time $t\sim\mathcal U(0,1)$ and noise $\epsilon$. Weight the usual denoising
or flow-matching loss by $w(X)$. Since $p=wq$, this is equivalent to
training on the target distribution:
\begin{equation}
 \begin{aligned}
 \mathcal L(\theta;w)
 &=\E_{X\sim q,\,t,\epsilon}[w(X)\ell_\theta(X,t,\epsilon)]\\
 &=\E_{X\sim p,\,t,\epsilon}[\ell_\theta(X,t,\epsilon)].
 \end{aligned}
 \label{eq:ftfc-native-population}
\end{equation}
Here $\ell_\theta$ is the standard denoising or flow-matching loss
\citep{zhang2025energyweighted}:
\begin{equation}
 \begin{aligned}
 \ell_\theta(X,t,\epsilon)&=\|m_\theta(Y_t,t)-A_t\|^2,\\
 (Y_t,A_t)&=
 \begin{cases}
 (a_tX+\sigma_t\epsilon,\ \epsilon),&\text{diffusion},\quad \epsilon\sim\mathcal N(0,I),\\
 ((1-t)\epsilon+tX,\ X-\epsilon),&\text{flow},\quad \epsilon\sim p_0.
 \end{cases}
 \end{aligned}
 \label{eq:ftfc-native-labels}
\end{equation}
The respective optimal predictors are the target noise predictor and
velocity field, since conditional regression gives
\begin{equation}
 m_p(y,t)=\E_p[A_t\mid Y_t=y]
       =\frac{\E_q[w(X)A_t\mid Y_t=y]}{\E_q[w(X)\mid Y_t=y]}.
 \label{eq:ftfc-conditional-main}
\end{equation}
These identities require finite conditional moments, $\sigma_t>0$,
and a positive denominator. Fresh noise preserves the chosen source law; reweighting
cached generation trajectories generally does not.

Equation~\eqref{eq:ftfc-conditional-main} justifies the weighted
regression used by FTFC. For diffusion models, the same calibrated
endpoint distribution also admits an alternative realization through
score guidance, which we detail below.

\Needspace{12\baselineskip}
\begin{theorem}[Score correction under endpoint reweighting]
\label{thm:ftfc-weight-guidance}
Let $w^\star\geq0$, $\E_qw^\star=1$ be recovered by
\eqref{eq:ftfc-functional-fdual}, and set $p^\star=w^\star q$.
For $Y_t=a_tX+\sigma_t\epsilon$ with independent $\epsilon\sim\mathcal N(0,I)$
and $\sigma_t>0$, denote the noisy densities by
$\widetilde q_t,\widetilde p_t^\star$ and define
$h_t(y)=\E_q[w^\star(X)\mid Y_t=y]$.
Assume $0<h_t<\infty$, $\E_{t,\widetilde q_t}|\log h_t|<\infty$,
and differentiation under the integral is valid. Writing
$s_q=\nabla_y\log\widetilde q_t$ and
$\epsilon_q=\E_q[\epsilon\mid Y_t=y]$ (and analogously for $p^\star$), we have
\begin{equation}
 \begin{aligned}
 \widetilde p_t^\star(y)&=h_t(y)\widetilde q_t(y),\\
 s_{p^\star}(y,t)&=s_q(y,t)+\nabla_y\log h_t(y),\\
 \epsilon_{p^\star}(y,t)&=\epsilon_q(y,t)-\sigma_t\nabla_y\log h_t(y).
 \end{aligned}
 \label{eq:ftfc-bridge-main}
\end{equation}
Over integrable scalar functions, the in-sample loss has the unique
minimizer
\begin{equation}
 \boxed{\begin{aligned}
 \mathcal L_{\rm IGL}(\ell;w^\star)
 &:=\E_{t,X\sim q,\epsilon}[w^\star(X)e^{-\ell(Y_t,t)}+\ell(Y_t,t)],\\
 \ell^\star(y,t)&=\log h_t(y),\qquad
 \nabla_y\ell^\star=s_{p^\star}-s_q.
 \end{aligned}}
 \label{eq:ftfc-igl-main}
\end{equation}
\end{theorem}
The result holds for any normalized nonnegative weights, including
general-$f$ responses with zeros. Weighted denoising
\eqref{eq:ftfc-native-population} learns the same target predictor
directly; Algorithm~\ref{alg:method} uses this route.
Appendix~\ref{app:ftfc-overall} proves the theorem and gives the
corresponding correction for Gaussian flow interpolations.
Under exact realization of the weighted regression target, the fitted
generator recovers the selected endpoint law $p^\star=w^\star q$ and
therefore inherits its endpoint optimality; we formalize this result and
its regularity conditions in Appendix~\ref{app:target-realization}.

\begin{algorithm}[!tb]
\caption{Fenchel Tilt Flow Control (FTFC)}
\label{alg:method}
\begin{algorithmic}[1]
\Require Pretrained model $\theta_0$ with endpoint distribution $q$;
utility $\mathcal F_N$; divergence $f$; regularization $\alpha$;
sample bank size $N$; training steps $K$

\State Sample and cache endpoints $\{x_i\}_{i=1}^N\sim q$ and
evaluate their rewards and utility features

\Statex \textbf{Stage 1: Calibrate the target distribution}
\State Compute the supporting rewards $\{\widehat g_i\}_{i=1}^N$
from the utility
\Comment{Sec.~\ref{ftfc:sec:practical-calibration}}
\State Compute normalized density-ratio weights
\Comment{Eq.~\eqref{eq:ftfc-functional-fdual}}
\Statex \hspace{\algorithmicindent}$\displaystyle
\widehat w_i
\in
\partial f_+^*\!\left(
    \frac{\widehat g_i-\widehat\nu_{h_i}}{\alpha}
\right),
\qquad
\sum_{i\in I_h} a_i\widehat w_i=\rho_h$
\State Freeze $\widehat w$

\Statex \textbf{Stage 2: Fit the generator}
\State $\theta\gets\theta_0$
\For{$k=1,\ldots,K$}
    \State Sample a minibatch of cached endpoints $x_i\sim q$
    \State Sample fresh $t$ and noise $\epsilon$, and construct the native
input-target pair $(Y_t,A_t)$ as in~\eqref{eq:ftfc-native-labels}
    \Comment{Eq.~\eqref{eq:ftfc-native-labels}}
    \State $\displaystyle
L(\theta)\gets\frac{1}{B}\sum_{i\in\mathcal B}
\widehat w_i
\|m_\theta(Y_t^{(i)},t)-A_t^{(i)}\|^2$
    \Comment{Eq.~\eqref{eq:ftfc-native-population}}
    \State $\theta\gets\theta-\eta_k\nabla_\theta L(\theta)$
\EndFor
\State \Return $\theta$
\end{algorithmic}
\end{algorithm}

\section{Experiments}
\label{sec:experiments}
\subsection{Molecular generation on QM9}
\label{sec:experiments-qm9}

We adapt FlowMol \citep{dunn2024flowmol} on QM9
\citep{ramakrishnan2014qm9} using negative GFN1-xTB energy
\citep{friede2024dxtb}, with an upper-superquantile objective targeting
the best \(0.2\%\) of rewards. We compare FTFC with pretrained FlowMol,
expected-reward Adjoint Matching (AM; \citealt{domingo2025adjoint}),
FDC \citep{de2025flow}, and TFFT \citep{wang2026efficient} under the
shared ODE evaluation in Table~\ref{tab:qm9}.

\begin{table}[!t]
  \centering
  \caption{\textbf{QM9 / FlowMol molecular generation.}
  ODE evaluation with $50{,}000$ attempts per seed.
  Rewards are in Hartree; bold denotes the best observed value.}
  \label{tab:qm9}

  \begingroup
  \resultstablestyle
  \setlength{\tabcolsep}{2.2pt}
  \renewcommand{\arraystretch}{1.05}

  \begin{tabularx}{\linewidth}{
    @{}
    >{\raggedright\arraybackslash}p{0.13\linewidth}
    *{7}{>{\centering\arraybackslash}X}
    @{}
  }
    \toprule
    \textbf{Method}
    & \shortstack{\textbf{Train}\\\textbf{days} $\downarrow$}
    & \shortstack{\textbf{Mean}\\$R$ $\uparrow$}
    & \shortstack{\textbf{SQ}$_{.998}$\\$\uparrow$}
    & \shortstack{\textbf{Upper}\\$10\%$ $\uparrow$}
    & \shortstack{\textbf{Valid.}\\(\%) $\uparrow$}
    & \shortstack{\textbf{SA}\\$\downarrow$}
    & \shortstack{\textbf{Topology}\\$R_{\mathrm{valid}}$ $\uparrow$} \\
    \midrule

    Pretrained
    & $0.0$
    & $28.6$
    & $38.7$
    & $34.1$
    & $\bestvalue{96.4}$
    & $4.6$
    & $0.8$ \\

    AM
    & $0.5$
    & $28.4$
    & $38.1$
    & $33.5$
    & $96.4$
    & $4.7$
    & $0.7$ \\

    FDC
    & $1.4$
    & $28.6$
    & $38.4$
    & $33.6$
    & $95.1$
    & $\bestvalue{4.4}$
    & $0.6$ \\

    TFFT
    & $0.6$
    & $28.9$
    & $39.2$
    & $34.4$
    & $94.2$
    & $4.7$
    & $0.7$ \\

    \midrule
    \ourmethod{FTFC}
    & $\bestvalue{0.04}$
    & $\bestvalue{29.1}\uncertainty{0.1}$
    & $\bestvalue{39.4}\uncertainty{0.3}$
    & $\bestvalue{35.5}\uncertainty{0.1}$
    & $96.1\uncertainty{1.8}$
    & $\bestvalue{4.4}\uncertainty{0.1}$
    & $\bestvalue{0.8}\uncertainty{0.0}$ \\

    \bottomrule
  \end{tabularx}
  \endgroup
\end{table}

\paragraph{Tail reward and molecular geometry.}
FTFC attains the highest observed mean and both upper-tail rewards.
It also exceeds every fine-tuned baseline on topology reward, which
measures connectivity preservation and relaxation strain
\citep{kotani2026atomic}, while remaining competitive in synthetic
accessibility \citep{ertl2009synthetic}. Thus FTFC offers a better
observed reward--geometry trade-off than these fine-tuned alternatives.
Including sample-bank acquisition, FTFC takes about one hour per seed,
versus $12-35$ hours for the baselines.

These runs show that one fitting stage with fixed endpoint weights can
improve a nonlinear tail objective, without repeated online reward
optimization.
These comparisons use the recorded method budgets, with three FTFC
seeds and one run per baseline. Appendix~\ref{app:qm9} gives the
protocol, metric definitions, and timing scope; it also documents
TFFT's inactive tail-reward gradients in this reconstruction.

\subsection{Text-to-image generation}
\label{sec:experiments-images}
We evaluate FTFC on the Stable Diffusion v1.5 benchmark of \citet{wang2026efficient}. ImageReward is the sole fine-tuning reward: EXP-FT optimizes its expectation, while FDC and L-TFFT target lower-tail CVaR at $\beta=0.2$ with KL regularization. CLIP \citep{radford2021clip}, HPSv2.1 \citep{wu2023hpsv2}, and DreamSim \citep{fu2023dreamsim} are evaluation-only metrics. We generate ten images for each of 100 released prompts; Table~\ref{tab:stable-diffusion} reports the results, and Appendix~\ref{app:stable-diffusion} gives settings, metric definitions, and uncertainty estimates.
\begin{table}[!ht]
  \centering
  \caption{\textbf{Stable Diffusion v1.5 image generation.}
  Evaluation over $1{,}000$ images. Values are mean $\pm$ standard error
  across five seeds; bold denotes the best observed mean.}
  \label{tab:stable-diffusion}

  \begingroup
  \resultstablestyle
  \setlength{\tabcolsep}{2.2pt}
  \renewcommand{\arraystretch}{1.05}

  \begin{tabularx}{\linewidth}{
    @{}
    >{\raggedright\arraybackslash}p{0.14\linewidth}
    *{6}{>{\centering\arraybackslash}X}
    @{}
  }
    \toprule
    & \multicolumn{2}{c}{\textbf{ImageReward}}
    & \multicolumn{2}{c}{\textbf{Alignment}}
    & \textbf{Diversity}
    & \textbf{Train} \\
    \cmidrule(lr){2-3}
    \cmidrule(lr){4-5}
    \cmidrule(lr){6-6}
    \cmidrule(l){7-7}

    \textbf{Method}
    & \shortstack{\textbf{Mean}\\$\uparrow$}
    & \shortstack{\textbf{L-CVaR}$_{.2}$\\$\uparrow$}
    & \shortstack{\textbf{CLIP}\\$\uparrow$}
    & \shortstack{\textbf{HPSv2.1}\\$\uparrow$}
    & \shortstack{\textbf{DreamSim}\\var. $\uparrow$}
    & \shortstack{\textbf{days}\\$\downarrow$} \\
    \midrule

    Pretrained
    & $0.3\uncertainty{0.1}$
    & $-1.0\uncertainty{0.0}$
    & $0.28\uncertainty{0.00}$
    & $0.26\uncertainty{0.00}$
    & $0.33\uncertainty{0.01}$
    & $0.0$ \\

    EXP-FT
    & $0.7\uncertainty{0.0}$
    & $-0.5\uncertainty{0.0}$
    & $0.28\uncertainty{0.00}$
    & $\bestvalue{0.28}\uncertainty{0.00}$
    & $0.30\uncertainty{0.01}$
    & $3.4$ \\

    FDC
    & $0.7\uncertainty{0.1}$
    & $-0.5\uncertainty{0.0}$
    & $0.28\uncertainty{0.00}$
    & $0.27\uncertainty{0.00}$
    & $0.30\uncertainty{0.01}$
    & $3.6$ \\

    L-TFFT
    & $0.8\uncertainty{0.1}$
    & $-0.5\uncertainty{0.0}$
    & $0.28\uncertainty{0.00}$
    & $0.27\uncertainty{0.00}$
    & $0.30\uncertainty{0.01}$
    & $2.8$ \\

    \midrule
    \ourmethod{FTFC}
    & $\bestvalue{0.9}\uncertainty{0.1}$
    & $\bestvalue{-0.4}\uncertainty{0.0}$
    & $\bestvalue{0.28}\uncertainty{0.00}$
    & $0.27\uncertainty{0.00}$
    & $\bestvalue{0.35}\uncertainty{0.01}$
    & $\bestvalue{0.2}$ \\

    \bottomrule
  \end{tabularx}

  \par\vspace{1pt}
  {\footnotesize\color{black!65}
  Training time is in A100 GPU-days; evaluation details are given in
  Appendix~\ref{app:stable-diffusion}.}
  \endgroup
\end{table}\vspace{-2pt}
\paragraph{Reward improvement without diversity collapse.}
FTFC raises mean and lower-tail ImageReward while keeping DreamSim variance at the pretrained level; all fine-tuned baselines reduce it. Thus, the gain is not explained by concentrating probability on a narrow set of high-reward outputs. FTFC instead selects and fits a fixed redistribution of pretrained mass. EXP-FT and L-TFFT improve reward with a larger diversity shift, whereas FTFC obtains this trade-off in one frozen-weight fitting stage.
\paragraph{Qualitative comparison.}
Figure~\ref{fig:stable-diffusion-qualitative} shows the lowest-, median-, and highest-reward samples for the prompt \textit{footage of an astronaut in a tropical beach}. Fine-tuning improves the worst samples but can alter the output distribution: L-TFFT, for example, favors saturated, fantastical skies despite high ImageReward. FTFC improves low-reward samples while retaining greater visual variation, consistent with its DreamSim diversity. Appendix~\ref{app:stable-diffusion} provides details.
\begin{figure}[!ht]
  \centering
  \includegraphics[width=\linewidth]{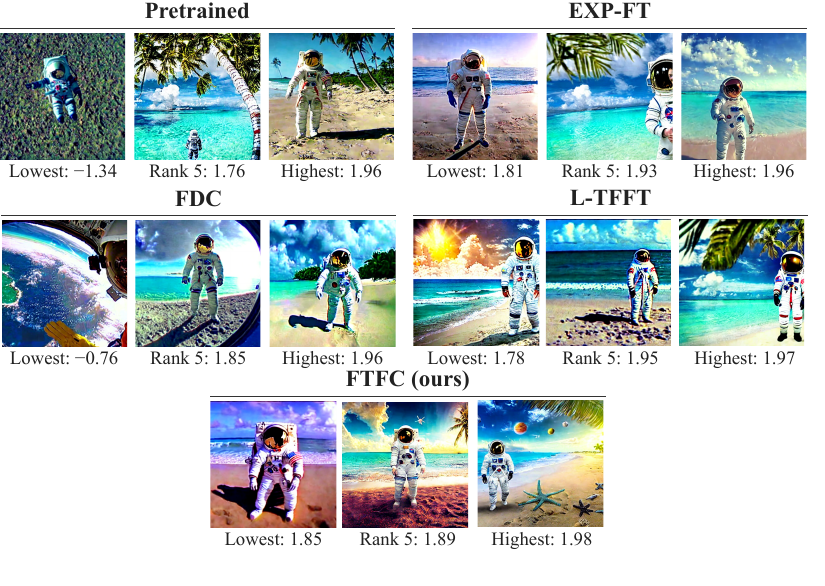}\vspace{-6pt}
  \caption{\textbf{Qualitative Stable Diffusion v$1.5$ comparison.}
  The prompt is \textit{footage of an astronaut in a tropical beach}.
  Each method panel shows ranks $1$ (lowest), $5$, and $10$ (highest), corresponding to ImageReward score. See Section \ref{sec:experiments-images}.}
  \label{fig:stable-diffusion-qualitative}
\end{figure}

\subsection{Molecular design on GEOM-Drugs}
\label{sec:experiments-molecules}
We evaluate FlowMol3 \citep{dunn2025flowmol3} on GEOM-Drugs \citep{axelrod2022geom}, independently reconstructing the energy-guided setting of \citet{wang2026efficient}. FDC, R-TFFT, and FTFC target the upper $10\%$ of negative GFN1-xTB rewards; EXP-FT uses AM to maximize expected reward. All methods share the backbone, reward oracle, and three-seed evaluation in Table~\ref{tab:geomdrugs}.
\paragraph{Preserving molecular quality during adaptation.}
FTFC has the highest observed graph validity and topology reward, with energy rewards near pretrained and above EXP-FT and R-TFFT. FDC attains higher mean and tail energy rewards but lowers topology reward through more protocol failures and greater relaxation strain, showing that higher energy reward can accompany worse geometry. FTFC avoids the larger degradation of the fine-tuned baselines; its small edge over pretrained is not statistically established. Figure~\ref{fig:geomdrugs-random-samples} shows one threshold-selected molecular panel, with further examples and protocol details in Appendix~\ref{app:geomdrugs-random-samples}.

\begin{table}[!ht]
  \centering
  \caption{\textbf{GEOM-Drugs / FlowMol3 molecular design.}
  Evaluation over three seeds and $2{,}000$ molecules per seed.
  Bold denotes the best observed mean.}
  \label{tab:geomdrugs}

  \begingroup
  \resultstablestyle
  \setlength{\tabcolsep}{2.2pt}
  \renewcommand{\arraystretch}{1.05}

  \begin{tabularx}{\linewidth}{
    @{}
    >{\raggedright\arraybackslash}p{0.17\linewidth}
    *{5}{>{\centering\arraybackslash}X}
    @{}
  }
    \toprule
    \textbf{Method}
    & \shortstack{\textbf{Mean}\\$R$ $\uparrow$}
    & \shortstack{\textbf{R-CVaR}\\$_{0.9}$ $\uparrow$}
    & \shortstack{\textbf{Valid.}\\(\%) $\uparrow$}
    & \shortstack{\textbf{SA}\\$\downarrow$}
    & \shortstack{\textbf{Topology}\\$R_{\mathrm{valid}}$ $\uparrow$} \\
    \midrule

    Pretrained
    & $65.9\uncertainty{0.1}$
    & $94.9\uncertainty{0.2}$
    & $99.8\uncertainty{0.0}$
    & $7.5\uncertainty{0.01}$
    & $0.95\uncertainty{0.0}$ \\

    EXP-FT
    & $67.9\uncertainty{0.8}$
    & $94.7\uncertainty{1.5}$
    & $99.8\uncertainty{0.2}$
    & $7.3\uncertainty{0.03}$
    & $0.93\uncertainty{0.0}$ \\

    FDC
    & $69.2\uncertainty{2.4}$
    & $\bestvalue{97.9}\uncertainty{1.9}$
    & $99.6\uncertainty{0.2}$
    & $7.2\uncertainty{0.01}$
    & $0.93\uncertainty{0.0}$ \\

    R-TFFT
    & $70.6\uncertainty{0.8}$
    & $90.8\uncertainty{0.9}$
    & $99.8\uncertainty{0.2}$
    & $7.3\uncertainty{0.03}$
    & $0.93\uncertainty{0.0}$ \\

    \midrule
    \ourmethod{FTFC}
    & $\bestvalue{75.3}\uncertainty{0.2}$
    & $95.0\uncertainty{0.3}$
    & $\bestvalue{99.9}\uncertainty{0.1}$
    & $\bestvalue{7.1}\uncertainty{0.01}$
    & $\bestvalue{0.97}\uncertainty{0.0}$ \\

    \bottomrule
  \end{tabularx}

  \par\vspace{2pt}
  {\footnotesize\color{black!65}
  Mean $\pm$ standard deviation over three seeds. More details in Appendix~\ref{app:geomdrugs}.}

  \endgroup
\end{table}

\paragraph{Ablating the $f$-divergence penalty.}
We vary only the $f$-divergence used to regularize FTFC molecular target Eq. \eqref{eq:setup-endpoint-objective} and Eq.\eqref{eq:ftfc-feature-dual}, while keeping the utility, sample bank, and generator-fitting procedure fixed. Table~\ref{tab:f-penalty-ablation} reports mean reward, upper-tail CVaR, and topology $R_{\mathrm{valid}}$. Mean reward captures typical energetic quality, while upper-tail CVaR asks whether the method still places mass on the rare, especially favorable molecules that matter in screening. Topology reward is a necessary complement: a high-reward conformation is practically useful only when its bonding pattern and relaxed geometry remain sound.

\begin{wrapfigure}{r}{0.52\linewidth}
  \vspace{-6pt}
  \centering
  \includegraphics[width=\linewidth]{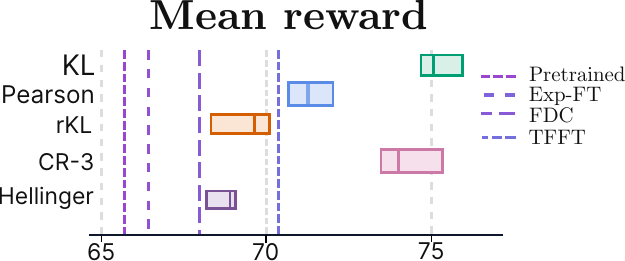}
  \caption{Reward comparison for the $f$-divergence ablation on GEOM-Drugs. Dashed markers denote the baseline scores.}
  \label{fig:geomdrugs-ablation}
  \vspace{-6pt}
\end{wrapfigure}

Table~\ref{tab:f-penalty-ablation} shows that the choice of $f$-divergence has its clearest effect on the upper tail. Cressie-Read-3 achieves the strongest upper-tail CVaR, followed by squared Hellinger and KL, while Half-Pearson and Reverse KL trail behind. Thus, when the objective is to enrich the pool of especially favorable molecules, Cressie--Read-3 is the preferred choice: it combines the best tail score with a competitive mean reward and topology validity. KL instead gives the highest mean reward and topology validity, but its upper-tail CVaR is lower than Cressie--Read-3. Reverse KL is weakest on both mean reward and upper-tail CVaR, while Half-Pearson and squared Hellinger remain close to KL in the tail. Overall, the results show that the $f$-divergence is a meaningful control knob for prioritizing elite-molecule quality: Cressie--Read-3 favors the upper tail, whereas KL favors average reward and validity. The appendix gives the divergence-specific weight maps and their connection to the FTFC dual.
\begin{table}[!t]
\centering
  \caption{Ablation of the $f$-divergence penalty in FTFC on GEOM-Drugs. Values are mean $\pm$ standard deviation over three seeds.}
  \label{tab:f-penalty-ablation}
  \begingroup
  \small
  \setlength{\tabcolsep}{4pt}
  \renewcommand{\arraystretch}{1.05}
  \begin{tabularx}{\linewidth}{
    @{}
    >{\raggedright\arraybackslash}p{0.28\linewidth}
    *{3}{>{\centering\arraybackslash}X}
    @{}
  }
    \toprule
    \textbf{$f$ penalty}
    & \shortstack{\textbf{Mean}\\reward $\uparrow$}
    & \shortstack{\textbf{Upper-tail}\\CVaR $\uparrow$}
    & \shortstack{\textbf{Topology}\\$R_{\mathrm{valid}}$ $\uparrow$} \\
    \midrule
    KL
    & $\bestvalue{75.3}\uncertainty{0.2}$
    & $95.0\uncertainty{0.1}$
    & $\bestvalue{0.97}\uncertainty{0.0}$ \\
    Half-Pearson
    & $71.4\uncertainty{0.4}$
    & $94.9\uncertainty{0.1}$
    & $0.95\uncertainty{0.0}$ \\
    Reverse KL
    & $68.0\uncertainty{0.4}$
    & $93.8\uncertainty{0.1}$
    & $0.95\uncertainty{0.0}$ \\
    Cressie-Read-3
    & $73.2\uncertainty{0.7}$
    & $\bestvalue{96.3}\uncertainty{0.1}$
    & $0.96\uncertainty{0.0}$ \\
    Squared Hellinger
    & $69.0\uncertainty{0.2}$
    & $95.1\uncertainty{0.3}$
    & $0.95\uncertainty{0.0}$ \\
    \bottomrule
  \end{tabularx}
  \endgroup
\end{table}
\begin{figure}[!htbp]  
\centering  
\includegraphics[width=0.68\linewidth]{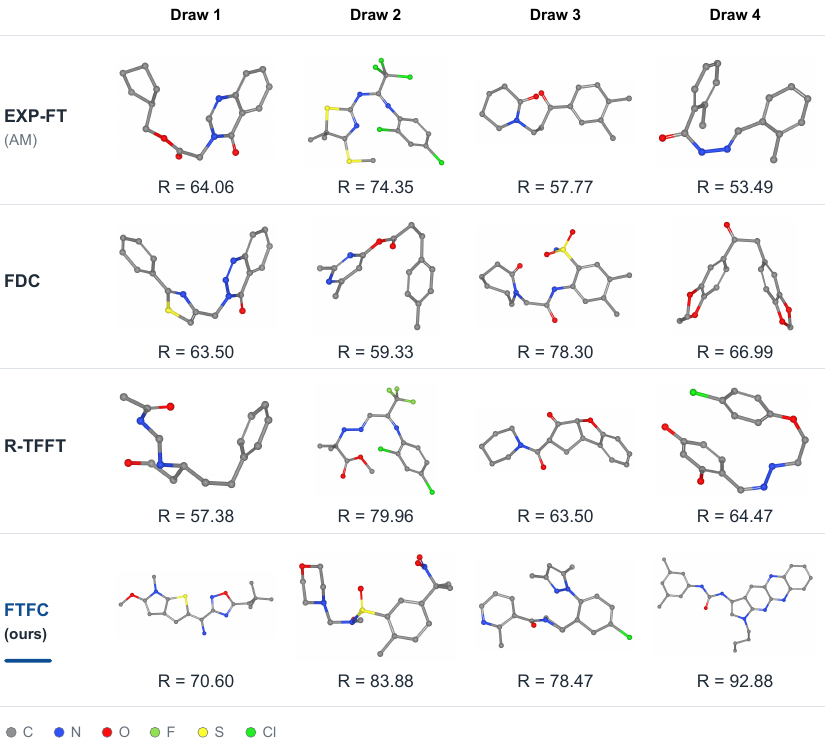}  
\caption{\textbf{Molecular examples with one threshold-selected panel.}  
FTFC Draw~2 is the first new sample with measured reward \(R\geq80\); the other panels show unselected draws. This selection is illustrative and does not establish comparative performance.  
Labels give the measured, unscaled GFN1 reward \(R=-E_{\mathrm{GFN1}}\). Appendix~\ref{app:geomdrugs-random-samples} specifies the protocol and number of candidates.}  
\label{fig:geomdrugs-random-samples}
\end{figure}

\FloatBarrier

\paragraph{Molecular fitting in FTFC.}
FTFC samples complete molecules at their target masses, preserving coordinates, atom types, formal charges, and bonds. FlowMol3 retains its native time and modality weights, categorical masks, fake-atom handling, and stochastic self-conditioning; fresh-prior fitting uses the native objective without reward-gradient or validity losses. Thus frozen endpoint weights define the target without changing the per-molecule weighting.

\section{Conclusion \& Limitations}
\label{sec:conclusion}
We introduced Fenchel Tilt Flow Control (FTFC), a distribution-correction approach to utility-based generative fine-tuning. The key idea is to separate \emph{which distribution to generate} from \emph{how to generate it}: FTFC first optimizes normalized density-ratio weights defining the target distribution, then freezes them and fits the generator with its native denoising or flow-matching objective. This enables nonlinear utilities and general $f$-divergence penalties without repeated reward optimization during generator training. Across image and molecular generation tasks, FTFC achieves strong performance with substantially lower training time. 

FTFC remains limited by the coverage of the pretrained distribution and its assumptions on finite sample bank. Moreover, the theoretical target-optimization guarantees do not account for finite-bank estimation or imperfect generator fitting, and the general dual requires suitable concavity and regularity conditions. Extending FTFC beyond fixed sample banks and to broader nonconcave utilities are promising directions.

\FloatBarrier

\section*{Reproducibility Statement}

The Appendix contains the details required to reproduce our results, including dataset and checkpoint information, preprocessing, evaluation protocols, model and optimizer configurations, all hyperparameters, training budgets, random seeds, and metric definitions. It also gives the relevant derivations, background, details on baselines, and implementation-specific settings used in the experiments. 

\section*{Ethics Statement}

This work uses publicly available datasets and pretrained models and does not involve human subjects, personal data, or human-subject experiments. The method is intended for research on generative modeling.

\section*{AI Use Statement}

AI tools were used only to polish the manuscript text and improve grammar and readability. All technical content, experiments, and final wording were reviewed and verified by the authors.

\ificlrhascitations  
\clearpage  
\bibliography{references}  
\bibliographystyle{iclr2027_conference}
\fi
\appendix
\section{Additional Background Derivations}
\label{app:proofs}
This appendix derives the utility dual, its finite-bank form, and the
transfer from endpoint weights to a generative model.

\subsection{Utility duality and weight recovery}
\label{app:ftfc-new-proofs}
\paragraph{From the original objective to the dual bound.}
\label{app:ftfc-reward-bound}
Let $p\ll q$ with density ratio $w=dp/dq\geq0$, $\E_qw=1$, and
assume the displayed terms are finite. Change of measure and the
definition of the $f$-divergence give, respectively,
\[
 \begin{aligned}
 \E_{wq}[g]&=\int g(x)\,dp(x)
           =\int g(x)w(x)\,dq(x)=\E_q[wg],\\
 D_f(wq\|q)&=\int f\!\left(\frac{dp}{dq}(x)\right)dq(x)
           =\E_q[f(w)].
 \end{aligned}
\]
The supremum defining $C(g)$ includes the choice $p'=wq$, so
\[
 C(g)=\sup_{p'}\{\mathcal F(p')-\E_{p'}[g]\}
 \geq\mathcal F(wq)-\E_{wq}[g].
\]
Substituting into the original objective~\eqref{eq:go-objective} yields
\[
 \begin{aligned}
 \mathcal G(wq)
 &=\mathcal F(wq)-\alpha D_f(wq\|q)\\
 &=\mathcal F(wq)-\alpha\E_q[f(w)]\\
 &\leq C(g)+\E_{wq}[g]-\alpha\E_q[f(w)]\\
 &=C(g)+\E_q[wg]-\alpha\E_q[f(w)]\\
 &=C(g)+\E_q[wg-\alpha f(w)].
 \end{aligned}
\]
Next, introduce a scalar multiplier $\nu$ for normalization. Since
$\E_qw=1$, adding $\nu(1-\E_qw)=0$ leaves the bound unchanged:
\[
 \begin{aligned}
 L(w,g,\nu)
 &:=C(g)+\E_q[wg-\alpha f(w)]+\nu(1-\E_qw)\\
 &=C(g)+\nu+\E_q[(g-\nu)w-\alpha f(w)].
 \end{aligned}
\]
For fixed $g$ and $\nu$, maximize the integrand over a nonnegative
scalar weight $v$ at each $x$. Factoring out $\alpha>0$ gives
\[
 \begin{aligned}
 \sup_{v\geq0}\{(g(x)-\nu)v-\alpha f(v)\}
 &=\alpha\sup_{v\geq0}
   \left\{\frac{g(x)-\nu}{\alpha}v-f(v)\right\}\\
 &=\alpha f_+^*\!\left(\frac{g(x)-\nu}{\alpha}\right),
 \qquad f_+^*(u):=\sup_{v\geq0}\{uv-f(v)\}.
 \end{aligned}
\]
The value at $v=w(x)$ cannot exceed this supremum. Taking expectation
under $q$ therefore yields
\[
 \begin{aligned}
 \mathcal G(wq)&\leq L(w,g,\nu)\\
 &\leq C(g)+\nu+\alpha\E_q
    f_+^*\!\left(\frac{g-\nu}{\alpha}\right)
    =:D_f(g,\nu).
 \end{aligned}
\]
For proper closed convex $f$, equality in the pointwise bound holds when
$w(x)\in\partial f_+^*((g(x)-\nu)/\alpha)$; imposing $\E_qw=1$ gives the
response in~\eqref{eq:ftfc-functional-fdual}. The bound holds for every
finite $(g,\nu)$; equality of optimized primal and dual requires the
regularity conditions below.

For the exact population envelope in~\eqref{eq:ftfc-envelope}, use the
$L^1(q)$--$L^\infty(q)$ pairing and a proper, concave, norm-upper-semicontinuous
extension of $w\mapsto\mathcal F(wq)$; concave biconjugacy then applies.

\paragraph{Response to a fixed reward.}
Under strong duality for normalization, the optimized response value is
\begin{equation}
 \begin{aligned}
 T_f(g)&:=\sup_{w\geq0,\,\E_qw=1}
          \{\E_q[wg]-\alpha\E_qf(w)\}\\
 &=\inf_\nu\{\nu+\alpha\E_qf_+^*((g-\nu)/\alpha)\}.
 \end{aligned}
 \label{eq:ftfc-fixed-response}
\end{equation}
Integrated Fenchel equality recovers the normalized weights in
\eqref{eq:ftfc-functional-fdual}. Write $u=(g(x)-\nu)/\alpha$.
For the strictly convex generators in
Table~\ref{tab:ftfc-divergence-responses}, a positive maximizer of
$uv-f(v)$ satisfies $f'(v)=u$. Inverting $f'$ on its range gives the
interior response; the table also accounts for zero weights and the
admissible domain of $u$.

\begin{table}[htbp]
 \centering
 \small
 \setlength{\tabcolsep}{5pt}
 \renewcommand{\arraystretch}{1.4}
 \caption{Divergence generators, derivatives for $w>0$, and scalar
 weight responses $w(u)=\arg\max_{v\geq0}\{uv-f(v)\}$,
 where $u=(g(x)-\nu)/\alpha$ and $(a)_+=\max\{a,0\}$.}
 \label{tab:ftfc-divergence-responses}
 \fbox{\begin{tabular}{@{}lcccc@{}}
  \toprule
  Divergence & $f(w)$ & $f'(w)$ & $w(u)$ & Domain of $u$ \\
  \midrule
  KL & $w\log w-w+1$ & $\log w$ & $e^u$ & $\mathbb R$ \\
  Half-Pearson & $\tfrac12(w-1)^2$ & $w-1$ & $(1+u)_+$ & $\mathbb R$ \\
  Reverse KL & $-\log w+w-1$ & $1-1/w$ & $(1-u)^{-1}$ & $u<1$ \\
  \bottomrule
 \end{tabular}}
\end{table}
Half-Pearson permits $w=0$ when $u\leq-1$. Reverse KL has
$f(0)=+\infty$ and requires $u<1$ for a finite maximizing weight.
In every row, $\nu$ is chosen so that $\E_q[w]=1$.

For KL with finite $Z_g$, $f'(w)=\log w$ gives
$w_g(x)=e^{(g(x)-\nu)/\alpha}$.
Normalization gives $1=e^{-\nu/\alpha}\E_q[e^{g/\alpha}]$, hence
\begin{equation}
 \begin{aligned}
 Z_g&=\E_{X\sim q}[e^{g(X)/\alpha}],\qquad \nu_g=\alpha\log Z_g,\\
 T_{\rm KL}(g)&=\alpha\log Z_g,\qquad
 w_g(x)=e^{g(x)/\alpha}/Z_g.
 \end{aligned}
 \label{eq:ftfc-dual-main}
\end{equation}
Taking $g=g_{p^\star}$ yields~\eqref{eq:ftfc-marginal-gibbs}.

\begin{proof}[Proof of Theorem~\ref{thm:ftfc-general-dual}]
\emph{Gap identity.} Put $u=(g-\nu)/\alpha$. Since $\E_qw=1$,
\begin{align*}
 D_f(g,\nu)-\mathcal G(wq)
 &=C(g)+\nu+\alpha\E_q f_+^*(u)-\mathcal F(wq)+\alpha\E_q f(w)\\
 &=\underbrace{C(g)+\E_q[wg]-\mathcal F(wq)}_{\geq0}
   +\alpha\underbrace{\E_q[f(w)+f_+^*(u)-uw]}_{\geq0}.
\end{align*}
The first term is the envelope gap and the second is the Fenchel--Young gap.
Both vanish exactly when the envelope is tight and
$u\in\partial(f+I_{[0,\infty)})(w)$ almost surely, proving optimality and
recovering the weights by conjugacy.

\emph{Finite-bank attainment.} Let $a_i>0$, $b\in\Delta_N$, and
$J(b)=\alpha\sum_i a_i f(b_i/a_i)$. Continuity on the compact simplex
gives a maximizer $b^*$ of $\mathcal F_N-J$; strict convexity of $J$
makes it unique. Write $\ell_0=f'_+(0)$.
If $\ell_0=-\infty$, this maximizer is interior. Indeed, for
$b_\varepsilon=(1-\varepsilon)b^*+\varepsilon a$,
\[
 \frac{\mathcal F_N(b_\varepsilon)-\mathcal F_N(b^*)}{\varepsilon}
 \geq\mathcal F_N(a)-\mathcal F_N(b^*),\qquad
 \frac{J(b_\varepsilon)-J(b^*)}{\varepsilon}\longrightarrow-\infty
\]
if any coordinate of $b^*$ is zero, contradicting optimality.

When $\ell_0$ is finite, extend $f$ below zero by its tangent at zero.
In the steep case, extend it below a positive ratio smaller than all
$b_i^*/a_i$. This yields a finite differentiable convex penalty
$\bar J$ agreeing with $J$ near $b^*$. The convex subgradient sum rule gives
\begin{align*}
 0&\in\partial(-\mathcal F_N+I_{\Delta_N})(b^*)+\nabla\bar J(b^*),\\
 \mathcal F_N(b)&\leq\mathcal F_N(b^*)+g^{*\top}(b-b^*),
 \qquad b\in\Delta_N,
\end{align*}
where $g_i^*=\alpha f'(b_i^*/a_i)$ for positive ratios and
$g_i^*=\alpha\ell_0$ at zero. Thus, with $\nu^*=0$,
\[
 C_N(g^*)+g^{*\top}b^*=\mathcal F_N(b^*),\qquad
 g_i^*/\alpha\in\partial(f+I_{[0,\infty)})(b_i^*/a_i).
\]
Both residuals vanish, establishing dual attainment and strong duality.
The argument also holds on a product of scaled simplices for fixed
condition masses. For a barrier such as reverse KL, compactness,
lower semicontinuity, and $f(0)=+\infty$ give an interior optimum;
the same local argument applies.

\emph{KL specialization.} With $f(w)=w\log w-w+1$,
$f_+^*(u)=e^u-1$, and normalized $w_g=e^{g/\alpha}/Z_g$,
\[
 \E_q[f(w)+f_+^*(\log w_g)-w\log w_g]
 =\E_q[w\log(w/w_g)-w+w_g]=\DKL(wq\|w_gq).
\]
Consequently,
\begin{equation}
 \begin{aligned}
 D_{\rm KL}(g)-\mathcal G(p)
 &=C(g)+\E_pg-\mathcal F(p)+\alpha\DKL(p\|w_gq),\\
 D_{\rm KL}(g)&=C(g)+\alpha\log Z_g.
 \end{aligned}
 \label{eq:ftfc-kl-certificate}
\end{equation}
At an optimal supporting reward, $w_gq=p^*$; hence for any finite
candidate dual value $D$,
\[
 \alpha\DKL(p\|p^*)\leq\mathcal G(p^*)-\mathcal G(p)\leq D-\mathcal G(p).
\]
\end{proof}

\paragraph{Sufficient population conditions for the structured dual.}
Let $r,\phi$ be bounded, $\Psi:\mathbb R^k\to\mathbb R$ finite
continuous convex, and let $f$ satisfy the finite-bank assumptions of
Theorem~\ref{thm:ftfc-general-dual}, with $f(1)=0$.
Then~\eqref{eq:ftfc-feature-dual} has zero duality gap and attained
primal and dual optima; the optimal ratio is unique. To see this, define
\[
 A(w)=\alpha\E_q f(w)-\E_q[wr]
      +I_{\{w\geq0,\,\E_qw=1\}}(w),\qquad Bw=\E_q[w\phi].
\]
The integral functional is closed convex, $B$ is continuous, and
$A(1)<\infty$. Continuity of $\Psi$ at $B1$ qualifies
Fenchel--Rockafellar duality \citep{rockafellar1970convex}:
\[
 \inf_w\{A(w)+\Psi(Bw)\}
     =\max_z\{-\Psi^*(z)-T_f(r-z^\top\phi)\}.
\]
The value is finite because rewards and attainable moments are bounded
and $\E_q f(w)\geq f(1)=0$. A minimizing sequence has bounded
$\E_q f(w)$; superlinearity gives uniform integrability and weak
compactness in $L^1(q)$. Weak lower semicontinuity then gives primal
attainment, and strict convexity gives uniqueness.

For any bounded score $s$, the conjugate response
$R(u)=\arg\max_{v\geq0}\{uv-f(v)\}$ is continuous, nondecreasing,
with $R(-\infty)=0$ and $R(+\infty)=+\infty$. Consequently
\[
 \E_q R((s-\nu)/\alpha)\longrightarrow
 \begin{cases}+\infty,&\nu\to-\infty,\\0,&\nu\to+\infty,\end{cases}
 \qquad\exists\nu:\ \E_q R((s-\nu)/\alpha)=1.
\]
At this root, integrated Fenchel equality proves~\eqref{eq:ftfc-fixed-response}
and scalar dual attainment. Let $D_{f,\Psi}(z,\nu)$ denote the objective
in~\eqref{eq:ftfc-feature-dual}. For any feasible $w$ and finite dual
value, expansion gives
\begin{equation}
 \begin{aligned}
 D_{f,\Psi}(z,\nu)-\mathcal G(wq)
 &=\Psi(m_w)+\Psi^*(z)-z^\top m_w\\
 &\quad+\alpha\E_q[f(w)+f_+^*(u)-uw],\\
 m_w&=\E_q[w\phi],\qquad u=(r-z^\top\phi-\nu)/\alpha.
 \end{aligned}
 \label{eq:ftfc-general-gap}
\end{equation}
Thus the two optimality conditions are the normalized conjugate
response and $z\in\partial\Psi(m_w)$.
For a proper closed convex extended-valued $\Psi$, it suffices that it be continuous at an
attainable finite-penalty moment, relative to the moment affine space.
Coverage entropy qualifies after removing feature coordinates that
vanish $q$-almost surely. These sufficient conditions cover KL and
half-Pearson; unbounded rewards and population reverse KL need separate
integrability and attainment conditions. The finite-bank results do
not require bounded population rewards.

\paragraph{Examples of fixed features.}
\label{app:ftfc-structured-examples}
Expected reward has $\Psi=0$ and $g=r$, so no feature vector is needed.
For nonlinear utilities, examples with $\gamma>0$ are:
\begin{center}
\begingroup
\setlength{\tabcolsep}{4pt}
\renewcommand{\arraystretch}{1.25}
\fbox{\begin{minipage}{\dimexpr\linewidth-2\fboxsep-2\fboxrule\relax}
\small
\begin{tabularx}{\linewidth}{@{}p{0.24\linewidth}p{0.32\linewidth}X@{}}
\textbf{Utility} & \textbf{Features $\phi(x)$} & \textbf{Penalty $\Psi(m)$}\\
\midrule
Coverage entropy & Region memberships & $\gamma\sum_jm_j\log m_j$\\
Moment matching & Chosen descriptors & $\frac\gamma2\|m-m_0\|^2$\\
D-optimal design ($r=0$) & $v(x)v(x)^\top$ & $-\log\det(M_0+m)$\\
\end{tabularx}
\end{minipage}}
\endgroup
\end{center}
For coverage, $\phi_j\geq0$ and $\sum_j\phi_j=1$, so $m_j$ is a
region's aggregate mass. The rings use 16 soft radial-basis memberships
(Appendix~\ref{app:toy2d-entropy}); marginal costs discourage
overrepresented regions. This is region entropy, not differential
entropy of the full density. Moment matching instead specifies desired
descriptor means $m_0$. The conjugates are given below.
For experimental design, $v(x)$ is a fixed experiment feature vector
and $M_p=M_0+\E_p[vv^\top]$, $M_0\succ0$. The utility is
$\log\det M_p$, with $g_p(x)=v(x)^\top M_p^{-1}v(x)$;
matrix inner products replace $z^\top\phi$
\citep[Section~7.5]{boyd2004convex}.

For $\gamma>0$ and a desired moment $m_0$, the structured dual uses
\[
 \begin{aligned}
 \Psi(m)&=\tfrac\gamma2\|m-m_0\|^2
 &\Longrightarrow\quad \Psi^*(z)&=z^\top m_0+\tfrac1{2\gamma}\|z\|^2,\\
 \Psi(m)&=\gamma\sum_jm_j\log m_j,\quad m\in\Delta_k
 &\Longrightarrow\quad \Psi^*(z)&=\gamma\log\sum_j e^{z_j/\gamma}.
 \end{aligned}
\]
At an optimum, $z^\star\in\partial\Psi(m_{p^\star})$.
For the entropy penalty, marginal costs are defined up to a common
additive constant on the simplex. Under KL, eliminating $\nu$ gives
\[
 D_{\rm KL,\Psi}(z)=\Psi^*(z)+\alpha\log\E_qe^{(r-z^\top\phi)/\alpha},
 \qquad \nabla D_{\rm KL,\Psi}(z)=\nabla\Psi^*(z)-\E_{w_{g_z}q}\phi.
\]
At an optimum, the moment implied by the conjugate equals the weighted
feature mean.

\paragraph{Utility examples.}
\label{app:ftfc-utility-catalog}
Table~\ref{tab:ftfc-utility-calibration} writes the utility families in FTFC
notation, using examples from \citet[Table~1]{de2025flow}. It distinguishes
the convex Moment Dual~\eqref{eq:ftfc-feature-dual}, scalar outer optimization,
and objectives with full-density dependence. The statements concern endpoint
calibration; native fitting and bank generalization remain separate.

\begin{table}[t]
\centering
\caption{\textbf{Utility functionals and their FTFC calibration.}
The first five rows use the Moment Dual~\eqref{eq:ftfc-feature-dual};
the last two calibrate a scalar tail threshold.}
\label{tab:ftfc-utility-calibration}
\begingroup
\small
\setlength{\tabcolsep}{0pt}
\renewcommand{\arraystretch}{1.15}
\begin{tabularx}{\linewidth}{@{}>{\raggedright\arraybackslash}p{0.34\linewidth}@{\hspace{6pt}}>{\raggedright\arraybackslash}p{0.33\linewidth}@{\hspace{6pt}}>{\raggedright\arraybackslash}X@{}}
\toprule
\textbf{Utility $\mathcal F(p)$} & \textbf{Representation} & \textbf{Reward / calibration}\\
\midrule
\textbf{Expected reward}\newline
$\E_p r$
& $\Psi=0$; no features
& $g^\star=r$\\[6pt]
\textbf{Moment matching}\newline
$\E_p r-\frac\gamma2\|m_p-m_0\|^2$
& $\phi$: chosen descriptors\newline
$\Psi(m)=\frac\gamma2\|m-m_0\|^2$
& $z^\star=\gamma(m^\star-m_0)$\newline
$g^\star=r-z^{\star\top}\phi$\\[6pt]
\textbf{Coverage entropy}\newline
$\E_p r+\gamma H(m_p)$
& $\phi$: region memberships\newline
$\Psi(m)=\gamma\sum_jm_j\log m_j$
& $z_j^\star=\gamma(1+\log m_j^\star)$\newline
$g^\star=r-z^{\star\top}\phi$\\[6pt]
\textbf{D-optimal design}\newline
$\log\det M_p$,\quad $r=0$
& $\phi(x)=v(x)v(x)^\top$\newline
$\Psi(m)=-\log\det(M_0+m)$
& $g^\star=v^\top M_{p^\star}^{-1}v$\\[6pt]
\textbf{Constraint barrier}$^\dagger$\newline
$\E_p r+\gamma\log(B-\E_p c)$
& $\phi(x)=c(x)$\newline
$\Psi(m)=-\gamma\log(B-m)$
& $z^\star=\gamma/(B-m^\star)$\newline
$g^\star=r-z^\star c$\\[4pt]
\midrule
\textbf{Lower-tail CVaR}\newline
$L_\tau(p)=\sup_\eta\{\eta+\E_p g_\eta^-\}$
& $g_\eta^-(x)=-(\eta-r(x))_+/\tau$
& Maximize $\eta+T_f(g_\eta^-)$\newline
over observed reward knots.\\[6pt]
\textbf{Upper-tail CVaR}\newline
$U_\tau(p)=\inf_\eta\{\eta+\E_p g_\eta^+\}$
& $g_\eta^+(x)=(r(x)-\eta)_+/\tau$
& Minimize $\eta+T_f(g_\eta^+)$\newline
over the scalar threshold.\\
\bottomrule
\end{tabularx}
\par\vspace{3pt}
\begin{minipage}{\linewidth}
\footnotesize
$m_p=\E_p\phi$, $m^\star=m_{p^\star}$, $\gamma>0$,
$H(m)=-\sum_jm_j\log m_j$, and
$M_p=M_0+\E_p[vv^\top]$ with $M_0\succ0$.
The tail mass is $\tau\in(0,1)$; $T_f(s)=\sup_{p\ll q}\{\E_p s-\alpha D_f(p\Vert q)\}$
is the optimized fixed-reward value.
$^\dagger$The barrier uses the maximization convention, with $\E_p c<B$.
Coverage entropy concerns fixed memberships, not the full continuous density.
\end{minipage}
\endgroup
\end{table}

\paragraph{Entropy and coverage.}
For a finite outcome space, $\phi_j(x)=\mathbf1\{x=x_j\}$ gives
$m_j=p(x_j)$, so $\Psi(m)=\gamma\sum_jm_j\log m_j$ represents
$\gamma H(p)$ exactly with $r=0$. Region indicators instead give
the entropy of region masses. In a continuous space,
$H(p)=-\int p(x)\log p(x)\,d\mu(x)$ depends on the entire density
relative to the reference measure $\mu$; its marginal reward is
$-1-\log p(x)$ where defined. The functional dual can apply under
its regularity assumptions, but fixed region features do not exactly
represent this differential entropy or supply its density values.

\paragraph{Experimental design.}
For any concave matrix criterion $s$, choose
$\phi(x)=v(x)v(x)^\top$ and
$\Psi(m)=-s(M_0+m)$, with matrix inner products in place of dot products.
Then $\mathcal F(p)=s(M_p)$ and
$g_p(x)=\langle\nabla s(M_p),v(x)v(x)^\top\rangle$ when differentiable.
The choices $s(M)=\log\det M$ and
$s(M)=-\operatorname{tr}(M^{-1})$ yield
$g_p(x)=v(x)^\top M_p^{-1}v(x)$ and
$g_p(x)=v(x)^\top M_p^{-2}v(x)$, respectively.
The criterion $s(M)=-\lambda_{\max}(M)$ is also concave and uses a
supergradient at repeated eigenvalues. We use $M_0\succ0$ to ensure
the domain of the log-determinant and inverse criteria; other fixed
offsets require checking that domain.

\paragraph{Constraint penalties and barrier signs.}
Any convex penalty of the mean cost, $\Psi(\E_p c)$, fits the moment
class. For the upper bound $\E_p c<B$, the maximization barrier is
$\mathcal F(p)=\E_p r+\gamma\log(B-\E_p c)$:
$\Psi(m)=-\gamma\log(B-m)$ is convex, and the marginal cost
$z=\gamma/(B-m)$ grows as the budget is approached.
This convention differs from the expression
$\E_p r-\gamma\log(\E_p c-B)$ printed in FDC's table.
The latter would require $\Psi(m)=\gamma\log(m-B)$, which is concave,
so it is not an instance of the convex Moment Dual.

\paragraph{Mean--variance reward.}
For $\mathcal F(p)=\E_p r-\gamma\operatorname{Var}_p(r)$,
$\phi=(r,r^2)$ gives $\Psi(m)=\gamma(m_2-m_1^2)$.
This representation is exact, but $\Psi$ is not convex. The marginal reward is
$g_p=r-\gamma r^2+2\gamma(\E_p r)r$; Equation~\eqref{eq:ftfc-feature-dual}
does not provide a convex calibration for it.
There is instead a scalar outer reduction:
\[
 \begin{aligned}
 \mathcal F(p)&=\sup_a\E_p[r-\gamma(r-a)^2],\\
 \sup_p\{\mathcal F(p)-\alpha D_f(p\Vert q)\}
 &=\sup_a T_f(r-\gamma(r-a)^2).
 \end{aligned}
\]
Indeed, $\E_p(r-a)^2=\operatorname{Var}_p(r)+(\E_p r-a)^2$;
the maximizing center is $a=\E_p r$, and the two suprema commute.
On a finite bank, $a$ can be restricted to the observed reward range.
The outer objective need not be concave, and the CVaR reward-knot
rule does not apply.

\textbf{Conditional mode objectives.}
Let fixed condition masses be $\rho_h>0$ and
$\bar p=\sum_h\rho_h p_h$. The negative-sign expression in FDC's
table is $-\sum_h\rho_h\DKL(p_h\Vert\bar p)$.
On finite outcome cells $A_j$, use the joint law $P(h,x)=\rho_h p_h(x)$
and features $\phi_{hj}(h',x)=\mathbf1\{h'=h,\,x\in A_j\}$.
Writing $m_{hj}=P(h,A_j)$ and $\bar m_j=\sum_hm_{hj}$, its discrete
version has $r=0$ and the convex penalty
\[
 \Psi(m)=\sum_{h,j}m_{hj}\log\frac{m_{hj}}{\rho_h\bar m_j},
 \qquad
 g_P(h,x)=-\log\frac{m_{hj}}{\rho_h\bar m_j}\quad(x\in A_j).
\]
Convexity follows from joint convexity of relative entropy with fixed
$\rho_h$; normalization uses one multiplier per condition.
For continuous conditionals, the exact functional retains their full
density ratios. The sign matters: maximizing the negative expression
reduces conditional separation, whereas maximizing positive mutual
information encourages distinct modes and reverses the convex penalty's
sign. The latter is not covered by the convex Moment Dual.

\textbf{Tail objectives.}
The lower- and upper-tail rows use the exact threshold representations
in \eqref{eq:ftfc-tails-main}, including fractional mass at atoms.
They require no moment-feature vector. Their outer optimization
directions differ: lower CVaR maximizes its threshold objective and
upper CVaR minimizes it, as derived in \eqref{eq:ftfc-tail-duals-main}.
Proposition~\ref{prop:ftfc-general-tail-knots} supplies the exact
finite-bank search for the lower tail.

\paragraph{Divergence choices in the molecular ablation.}
The $f$-divergence controls how FTFC converts the calibrated supporting reward into density-ratio weights.  Writing $s=(g^\star-\nu^\star)/\alpha$, FTFC's general dual gives $w^\star\in\partial f_+^\star(s)$, where the nonnegative conjugate enforces $w^\star\geq 0$ and the scalar $\nu^\star$ normalizes the target.  The ablation in Table~\ref{tab:f-penalty-ablation} therefore changes the shape of this reward-to-weight map while leaving the utility, sample bank, and generator-fitting stage unchanged.

For KL, $f(t)=t\log t-t+1$, so $f^\star(s)=e^s-1$ and FTFC recovers the familiar exponential tilt $w^\star=e^s$.  Reverse KL, $f(t)=-\log t+t-1$, instead gives $f^\star(s)=-\log(1-s)$ on $s<1$ and $w^\star=(1-s)^{-1}$.  Squared Hellinger, $f(t)=(\sqrt t-1)^2$, has $f^\star(s)=s/(1-s)$ on the same domain and yields $w^\star=(1-s)^{-2}$.  Thus Reverse KL and squared Hellinger both impose a finite-domain barrier: their dual arguments must remain below one, with the normalization multiplier selecting an admissible target.

The Half-Pearson choice uses $f(t)=\tfrac12(t-1)^2$.  Its unconstrained map is affine, $w^\star=1+s$; the nonnegative conjugate used by FTFC truncates this to $[1+s]_+$.  Cressie--Read-3 uses $f(t)=(t^3-3t+2)/6$, giving $w^\star=\sqrt{[1+2s]_+}$.  These two choices replace KL's exponential response with polynomial responses and can set sufficiently unfavorable samples to zero weight.  All five penalties are convex, normalized by $f(1)=0$, and fit the same Fenchel calibration and frozen-weight fitting result; the ablation tests the resulting target redistribution rather than a change in the generator objective.

\subsection{Finite-bank optimization}
\label{app:ftfc-method-details}\label{app:ftfc-support}\label{app:ftfc-fenchel}
Let $a_i>0$ be reference masses on cached endpoints $x_i$,
$\sum_i a_i=1$. For condition groups $I_h$, keep their reference
probabilities $\rho_h=\sum_{i\in I_h}a_i$ fixed. Target masses $b_i$
and ratios $w_i$ satisfy
\begin{equation}
 \begin{aligned}
 \mathcal C&=\{b\geq0:\sum_{i\in I_h}b_i=\rho_h\ \text{for every }h\},
 \qquad b_i=a_iw_i,\\
 \mathcal G_N(b)&=\mathcal F_N(b)-\alpha\sum_i a_i f(b_i/a_i).
 \end{aligned}
 \label{eq:ftfc-bank}
\end{equation}
Without conditioning, there is one group, so $\mathcal C$ is the simplex.
The group multipliers enter the Lagrangian as
$\sum_h\nu_h(\rho_h-\sum_{i\in I_h}b_i)$. Maximizing each ratio
$b_i/a_i\geq0$ independently yields
\begin{equation}
 \begin{aligned}
 D_{f,N}(g,\nu)&=C_N(g)+\sum_h\rho_h\nu_h
       +\alpha\sum_i a_i f_+^*((g_i-\nu_{h_i})/\alpha),\\
 C_N(g)&=\sup_{b\in\mathcal C}\{\mathcal F_N(b)-b^\top g\},\qquad
 b_i\in a_i\partial f_+^*((g_i-\nu_{h_i})/\alpha).
 \end{aligned}
 \label{eq:ftfc-empirical-dual}
\end{equation}
Choose $\nu$ so $b\in\mathcal C$. Theorem~\ref{thm:ftfc-general-dual}
applies with sums in place of expectations. For
$\mathcal F_N(b)=r^\top b-\Psi(\sum_i b_i\phi_i)$, substitute
$g_i=r_i-z^\top\phi_i$ and the envelope offset $\Psi^*(z)$.
For proper closed convex $\Psi$, the finite-dimensional qualification is
$b^\circ\in\operatorname{ri}\mathcal C$ with
$\sum_i b_i^\circ\phi_i\in\operatorname{ri}\operatorname{dom}\Psi$;
the preceding Fenchel argument then applies.

\paragraph{Solving for $z$ and $\nu$.}
\label{app:ftfc-solve-statistics}
For an unconditional bank, cache $x_i\sim q$, rewards $r_i$, and
features $\phi_i=\phi(x_i)$, with reference masses $a_i=1/N$. Replace $\E_q$ in the middle line of
\eqref{eq:ftfc-feature-dual} by $\sum_i a_i$ to obtain $D_N(z,\nu)$.
For each $z$, choose $\nu$ to normalize the conjugate response:
\[
 w_i(z)\in\partial f_+^*((r_i-z^\top\phi_i-\nu)/\alpha),
        \qquad\sum_i a_iw_i(z)=1.
\]
Under the response assumptions in Appendix~\ref{app:ftfc-new-proofs},
this is a monotone scalar equation; for KL, log-sum-exp gives the solution.
Minimize the resulting convex objective over $z$, whose gradient,
when defined, is
\[
 \nabla_z\!\left[\min_\nu D_N(z,\nu)\right]
       =\nabla\Psi^*(z)-\sum_i a_iw_i(z)\phi_i.
\]
Calibration matches the statistics implied by $z$ to those induced by
its weights. It uses cached endpoints, without updating the generator.
With fixed groups, use one normalizer per group as in
\eqref{eq:ftfc-empirical-dual}; the moment gradient is unchanged.

For KL, $Z_{z,h}=\sum_{i\in I_h}(a_i/\rho_h)
e^{(r_i-z^\top\phi_i)/\alpha}$ eliminates the normalizers:
\begin{equation}
 \begin{aligned}
 D_N(z)&=\Psi^*(z)+\alpha\sum_h\rho_h\log Z_{z,h},\\
 b_{z,i}&=a_i e^{(r_i-z^\top\phi_i)/\alpha}/Z_{z,h_i},
 \qquad m_b=\sum_i b_i\phi_i,\\
 D_N(z)-\mathcal G_N(b)
 &=\alpha\DKL(b\|b_z)+\Psi(m_b)+\Psi^*(z)-z^\top m_b.
 \end{aligned}
 \label{eq:ftfc-fenchel-gap}
\end{equation}
The gradient, when defined, is $\nabla\Psi^*(z)-m_{b_z}$.
The moment and entropy conjugates are in
Appendix~\ref{app:ftfc-structured-examples}. A global utility shares its $z$ or tail
threshold across groups; separate conditional utilities are different
objectives.

\paragraph{Tail rewards in the image and molecule benchmarks.}
\label{app:ftfc-tail-benchmarks}
CVaR uses an exact scalar-threshold representation instead of fixed
moment features. For tail mass $\tau$, its lower- and upper-tail
pseudo-rewards are
\begin{equation}
 g_c^{\rm lower}(x)=-\frac{(c-r(x))_+}{\tau},
 \qquad g_c^{\rm upper}(x)=\frac{(r(x)-c)_+}{\tau}.
 \label{eq:ftfc-benchmark-rewards}
\end{equation}
The threshold $c$ is calibrated using
\eqref{eq:ftfc-tail-duals-main}: upper-CVaR minimizes its threshold
objective; lower-CVaR maximizes it and is not covered by the concave
utility theorem. On a finite bank, the latter maximum is attained at
an observed reward (Appendix~\ref{app:ftfc-utilities}). These objectives
recover TFFT under KL \citep{wang2026efficient}.
\label{app:ftfc-benchmark-inputs}
In Stable Diffusion, calibration uses $r=100\,\mathrm{ImageReward}$,
$\alpha=1$, and lower-tail mass $0.2$. In QM9, it uses
$r=-E_{\mathrm{GFN1}}/100$, $\alpha=0.01$, and upper-tail mass $0.002$.
Both optimize a scalar threshold rather than moment features.
Prompt and atom-count probabilities remain fixed, respectively:
for group $h$, one normalizer enforces
$\sum_{i\in I_h}a_iw_i=\rho_h$, while the threshold is shared across groups.

\paragraph{CVaR calibration.}
\label{app:ftfc-utilities}
For tail mass $\tau\in(0,1)$, the lower and upper reward averages are
\begin{equation}
 L_\tau(p)=\sup_c\{c-\tau^{-1}\E_p(c-r)_+\},\qquad
 U_\tau(p)=\inf_c\{c+\tau^{-1}\E_p(r-c)_+\}.
 \label{eq:ftfc-tails-main}
\end{equation}
These definitions count fractional mass at threshold atoms.
Their threshold-dependent rewards are $-(c-r)_+/\tau$ and
$(r-c)_+/\tau$. Upper-CVaR is concave in the law; lower-CVaR is convex.
Using the fixed-reward response \eqref{eq:ftfc-fixed-response} gives
\begin{equation}
 \begin{aligned}
 V_U&=\min_c\{c+T_f((r-c)_+/\tau)\},\\
 V_L&=\max_c\{c+T_f(-(c-r)_+/\tau)\}.
 \end{aligned}
 \label{eq:ftfc-tail-duals-main}
\end{equation}
For the upper tail, restrict $c$ to the reward range and apply
convex--concave minimax under the finite-bank assumptions.
For the lower tail, commute two suprema. With KL,
\eqref{eq:ftfc-dual-main} recovers TFFT's threshold objectives
\citep{wang2026efficient}. The lower-tail objective need not be
concave; its exact finite-bank solution follows instead from
piecewise convexity.

\begin{proposition}[Lower-CVaR requires only reward knots]
\label{prop:ftfc-general-tail-knots}
Let $\mathcal B$ be a nonempty compact subset of the probability simplex,
$R$ lower semicontinuous, bounded below and finite somewhere on
$\mathcal B$, and $r_i$ finite. For $0<\tau<1$ and
$T_R(s)=\sup_{b\in\mathcal B}\{b^\top s-R(b)\}$,
\begin{equation}
 \max_{b\in\mathcal B}\{L_\tau(b)-R(b)\}
 =\max_{c\in\{r_1,\ldots,r_N\}}
       \{c+T_R(-(c\mathbf1-r)_+/\tau)\}.
 \label{eq:ftfc-general-tail-knots}
\end{equation}
An exact response at a maximizing knot attains the left side.
\end{proposition}
\begin{proof}
Commuting the two suprema gives
$H(c)=c+T_R(-(c\mathbf1-r)_+/\tau)$. The response value $T_R$ is
convex and $T_R(s+d\mathbf1)=T_R(s)+d$. Hence
\[
 H(c)=\begin{cases}
 c+T_R(0),&c\leq\min_i r_i,\\
 \text{convex on each }[r_{(j)},r_{(j+1)}],&\min_i r_i\leq c\leq\max_i r_i,\\
 (1-1/\tau)c+T_R(r/\tau),&c\geq\max_i r_i.
 \end{cases}
\]
The middle claim follows because each hinge is affine between knots.
The outer pieces increase and decrease, respectively, so a knot
maximizes $H$. Compactness and lower semicontinuity give an attained
response there; it maximizes the joint objective over $(b,c)$.
\end{proof}
For unconditional KL and sorted rewards $r_{(1)}\leq\cdots\leq r_{(N)}$,
each knot normalizer is
\[
 Z(c)=e^{-c/(\alpha\tau)}
       \sum_{r_i\leq c}a_i e^{r_i/(\alpha\tau)}
       +\sum_{r_i>c}a_i,\qquad H(c)=c+\alpha\log Z(c).
\]
Sorting plus log-domain prefix/suffix sums costs $O(N\log N)$.
For general $f$, the same knot search is exact but requires its own
normalization solves. At $\alpha=0$, solve the unregularized mass
problem; the scaled conjugate formulas do not apply.

\paragraph{Optional score optimization.}
If no compact dual is available, a neural score model can optimize
the full-bank objective through normalized scores. For positive masses,
$g_i=\partial\mathcal F_N/\partial b_i$ and
$u_i=g_i-\alpha f'(b_i/a_i)$, differentiation gives
\begin{equation}
 \begin{aligned}
 b_i&=\frac{a_i e^{s_\psi(x_i)}}
           {\sum_{j\in I_{h_i}}(a_j/\rho_{h_i})e^{s_\psi(x_j)}},\\
 \nabla_\psi\mathcal G_N
 &=\sum_i b_i(u_i-\bar u_{h_i})\nabla_\psi s_\psi(x_i),
 \qquad\bar u_h=\rho_h^{-1}\sum_{i\in I_h}b_i u_i.
 \end{aligned}
 \label{eq:ftfc-gradient}
\end{equation}
Indeed, within a group,
$\partial b_i/\partial s_j=b_i(\mathbf1_{i=j}-b_j/\rho_h)$;
different groups have zero cross-derivatives. Chunking this sum with
fixed full-bank coefficients preserves the gradient.
For a concave utility and a supergradient $g$ at $b$, the envelope is
tight there. Theorem~\ref{thm:ftfc-general-dual} therefore supplies a
bound on target suboptimality; under KL it is $\alpha\DKL(b\|b_g)$.
Here $b_{g,i}=a_i e^{g_i/\alpha}/\sum_{j\in I_{h_i}}(a_j/\rho_{h_i})e^{g_j/\alpha}$.
At $\alpha=0$ the bound is $\sum_h\rho_h\max_{i\in I_h}g_i-g^\top b$.
For nonconcave utilities these tangent expressions do not imply
global optimality.

\paragraph{Finite-bank approximation.}
The gap bounds optimization error for the discrete reference
$q_N=\sum_i a_i\delta_{x_i}$. It does not bound population or
generator error. Conditioning a bank on reward
eligibility replaces $q_h$ by $q_h(\cdot\mid\text{eligible})$;
changing deployment condition frequencies changes the joint law.
Optimization accuracy alone therefore does not establish sample
coverage or agreement with the population target. The eligibility
rules and conditioning groups for each benchmark are given in Appendix~\ref{app:implementation}.

\subsection{From endpoint weights to a generative model}
\label{app:target-realization}
\begin{corollary}[Optimality under exact target realization]
\label{thm:ftfc-overall}
Let $wq$ be $\varepsilon$-optimal for the endpoint relaxation. Assume finite
second moments of the endpoint and source, and an admissible exact
target predictor, obtained by native regression or diffusion guidance
as above. For flows, assume the ODE and continuity
equation have matching unique marginals. For diffusion, assume a
well-posed reverse process, initialized at the target's correct noisy
law and integrated exactly. Then
\begin{equation}
 \boxed{p_1^{\widehat\pi}=wq,\qquad
 0\leq\sup_{\pi\in\Pi}\mathcal G(p_1^\pi)
             -\mathcal G(p_1^{\widehat\pi})\leq\varepsilon.}
 \label{eq:ftfc-overall-guarantee}
\end{equation}
\end{corollary}
Conditional regression identifies the target dynamics, and uniqueness identifies
their endpoint law. The policy optimum is bounded by the endpoint relaxation;
finite-bank estimation, imperfect fitting, and numerical sampling remain
separate errors.

\label{app:ftfc-overall}
\begin{proof}[Proof of Theorem~\ref{thm:ftfc-weight-guidance}]
Fix $p=wq$ and $Y_t\sim K_t(\cdot\mid X)$. Write
$\widetilde q_t=K_t\#q$, $\widetilde p_t=K_t\#p$, and
$h_t(y)=\E_q[w(X)\mid Y_t=y]$. Assume finite conditional moments and
$0<h_t<\infty$ on the relevant support. For any test function $\varphi$ making the
following expectations finite, conditioning under the reference joint law gives
\begin{align*}
 \E_p\varphi(Y_t)
 &=\E_q[w(X)\varphi(Y_t)]
   =\E_{\widetilde q_t}[h_t(Y_t)\varphi(Y_t)],\\
 \E_p[A_t\varphi(Y_t)]
 &=\E_{\widetilde q_t}[\E_q[w(X)A_t\mid Y_t]\varphi(Y_t)].
\end{align*}
Dividing the second conditional density by the first proves
\eqref{eq:ftfc-conditional-main}; logarithmic differentiation of
$\widetilde p_t=h_t\widetilde q_t$ proves~\eqref{eq:ftfc-bridge-main}.
For the Gaussian kernel,
\[
 \nabla_y\log\widetilde p_t(y)
 =\E_p\!\left[-\frac{y-a_tX}{\sigma_t^2}\,\middle|\,Y_t=y\right]
 =-\frac{\E_p[\epsilon\mid Y_t=y]}{\sigma_t},
\]
and likewise under $q$, proving the noise-predictor correction.
For $d(y,t)=\ell(y,t)-\log h_t(y)$, conditioning the in-sample loss gives
\begin{align*}
 \mathcal L_{\rm IGL}(\ell;w)
 &=\E_{t,\widetilde q_t}[h_t e^{-\ell}+\ell],\\
 \mathcal L_{\rm IGL}(\ell;w)-\mathcal L_{\rm IGL}(\log h;w)
 &=\E_{t,\widetilde q_t}[e^{-d}-1+d]\geq0.
\end{align*}
Since $e^{-d}-1+d=0$ exactly when $d=0$, the unique minimizer is
$\ell^\star=\log h$ almost everywhere. Its smooth representative has
the score gradient in \eqref{eq:ftfc-igl-main}. Taking $w=w^\star$
proves the claim. This is the function-space argument of
\citet[Theorem~1]{mao2024diffusiondice}; neither neural optimization
convergence nor score-gradient accuracy follows from a small loss alone.
\end{proof}

\begin{proof}[Proof of Corollary~\ref{thm:ftfc-overall}]
For $I_t=(1-t)\epsilon+tX$, $X\sim p=wq$ independent of $\epsilon\sim p_0$,
let $\mu_t=\operatorname{Law}(I_t)$ and
$v_p(y,t)=\E_p[X-\epsilon\mid I_t=y]$. Change of measure and regression
orthogonality give
\begin{align*}
 \mathcal R(\pi;w)&=\E_{t,X\sim p,\epsilon\sim p_0}
                    \|\pi(I_t,t)-(X-\epsilon)\|^2,\\
 \mathcal R(\pi;w)-\mathcal R(v_p;w)
 &=\int_0^1\E_{\mu_t}\|\pi(Y_t,t)-v_p(Y_t,t)\|^2\,dt.
\end{align*}
The cross term vanishes because
$\E_p[X-\epsilon-v_p(I_t,t)\mid I_t]=0$. Since $v_p\in\Pi$, a global
minimizer $\widehat\pi$ equals $v_p$ under $dt\,\mu_t$ almost everywhere.
For every smooth compactly supported $\varphi$,
\[
 \frac{d}{dt}\E_p\varphi(I_t)
 =\E_p[\nabla\varphi(I_t)^\top(X-\epsilon)]
 =\int\nabla\varphi(y)^\top v_p(y,t)\,\mu_t(dy).
\]
Thus $\mu_t$ solves the continuity equation for both $v_p$ and
$\widehat\pi$, with $\mu_0=p_0$ and $\mu_1=p$.
The assumed uniqueness and consistency with ODE transport imply
$p_1^{\widehat\pi}=p$. Since every admissible policy is feasible for
the endpoint relaxation and $wq$ is $\varepsilon$-optimal for that relaxation,
\[
 0\leq\sup_{\pi\in\Pi}\mathcal G(p_1^\pi)
       -\mathcal G(p_1^{\widehat\pi})
 \leq\varepsilon.
\]
For diffusion, regression or Theorem~\ref{thm:ftfc-weight-guidance}
yields the exact target score. A well-posed reverse process,
initialized at the target's correct noisy law and integrated exactly,
then has endpoint $p$ and the same objective bound.
If the terminal noising coefficient $a_T=0$, this initial noisy law
is the common Gaussian $\mathcal N(0,\sigma_T^2I)$.
Otherwise, replacing it by a Gaussian introduces an additional approximation.
The guidance route also requires the exact reference score of the
same noised $q$ used for calibration.
\end{proof}

\paragraph{Fitting on a finite bank.}
For frozen bank masses, the training objective and its two equivalent
sampling forms are, with $t$ and $\epsilon$ sampled as in
\eqref{eq:ftfc-native-labels},
\begin{equation}
 \begin{aligned}
 \mathcal R_N(\theta;b)
 &=\sum_i b_i\E_{t,\epsilon}[\ell_\theta(x_i,t,\epsilon)]\\
 &=\E_{i\sim b,\,t,\epsilon}[\ell_\theta(x_i,t,\epsilon)]\\
 &=\E_{i\sim a,\,t,\epsilon}[(b_i/a_i)\ell_\theta(x_i,t,\epsilon)].
 \end{aligned}
 \label{eq:ftfc-native}
\end{equation}
Fresh native noise preserves the chosen source law. If $Q$ is the
original generation-path law with source $p_0$, weighting its paths
by $w(X_1)$ instead gives source marginal
$p_0(dx_0)\E_Q[w(X_1)\mid X_0=x_0]$, generally different from $p_0$.
If native preprocessing uses a noninvertible map $T$, fitting targets
$T_\#(wq)$; any exact utility guarantee must use that same observable.
The corollary concerns exact target predictors and exact sampling.
Finite-bank estimation, restricted adapters, and time discretization
remain separate errors.

\paragraph{Guidance for Gaussian flow interpolations.}
For $Y_t=a_tX+b_t\epsilon$, $\epsilon\sim\mathcal N(0,I)$ independent of $X$,
write $\bar x_q(y)=\E_q[X\mid Y_t=y]$ and $\bar x_p(y)=\E_p[X\mid Y_t=y]$.
For $a_t,b_t>0$, Gaussian differentiation gives
\begin{align*}
 \nabla\log\widetilde q_t(y)&=(a_t\bar x_q(y)-y)/b_t^2,\\
 \nabla\log h_t(y)&=a_t(\bar x_p(y)-\bar x_q(y))/b_t^2,\\
 \epsilon_p-\epsilon_q&=-a_t(\bar x_p-\bar x_q)/b_t
                      =-b_t\nabla\log h_t,\\
 v_p-v_q&=(\dot a_t-a_t\dot b_t/b_t)(\bar x_p-\bar x_q)
         =\left(\frac{\dot a_t b_t^2}{a_t}-b_t\dot b_t\right)\nabla\log h_t.
\end{align*}
Here $b_t$ plays the role of the diffusion noise scale $\sigma_t$.
For a Gaussian flow interpolation, the correction is
\begin{equation}
 v_p=v_q+\kappa_t\nabla\log h_t,\qquad
 \kappa_t=\frac{\dot a_t b_t^2}{a_t}-b_t\dot b_t.
 \label{eq:ftfc-flow-guidance}
\end{equation}
For $(a_t,b_t)=(t,1-t)$, $\kappa_t=(1-t)/t$ on $0<t<1$;
endpoint values require suitable limits. These identities assume finite
conditional moments and differentiation under the integral.
The reference predictor must correspond to this same corruption or
interpolation; arbitrary pretrained velocity fields cannot be corrected
by this formula without that compatibility
\citep{feng2025guidance,zhang2025energyweighted}.

\subsection{Reference-preserving regression by changing the training label}
\label{app:ftfc-residual-regression}

\paragraph{Scope.}
We analyze an optional fitting-loss modification that leaves Fenchel
calibration unchanged. It accounts for the gap between the pretrained
predictor and the conditional mean of generated endpoints under, e.g.,
classifier-free guidance. The modification is not used in the benchmarks;
the derivation alone does not imply better generation quality.

Retain the native pair $(Y_t,A_t)$ from
\eqref{eq:ftfc-native-labels}, a frozen normalized ratio $w=dp/dq$, and
a frozen predictor $m_0(y,t)$. Write
$\delta_\theta=m_\theta-m_0$ and let $\lambda>0$. All expectations below
use the same fresh corruption or interpolation conditional on the
endpoint, and have finite displayed second moments. Consider the
centered quadratic risk
\begin{equation}
 \mathcal L_\lambda(\theta)
 =\E_{q,t,\epsilon}\!\left[
   (w+\lambda)\|\delta_\theta\|^2
   +2(w-1)\langle\delta_\theta,m_0-A_t\rangle\right].
 \label{eq:ftfc-centered-risk}
\end{equation}
The quadratic coefficient is positive even where $w=0$. If $w\equiv1$ and
$m_\theta=m_0$ pointwise, each example has zero parameter gradient; ordinary
regression to fresh native labels need not have this property for an imperfect
teacher.

\begin{proposition}[An unweighted regression form]
\label{prop:ftfc-residual-regression}
Define the endpoint mixture and corrected native label
\begin{equation}
 \begin{aligned}
 s_\lambda&=\frac{p+\lambda q}{1+\lambda},
 &a_\lambda(X)&=\frac{w(X)-1}{w(X)+\lambda},\\
 \widetilde A_t&=m_0(Y_t,t)
             +a_\lambda(X)\bigl(A_t-m_0(Y_t,t)\bigr).
 \end{aligned}
 \label{eq:ftfc-residual-label}
\end{equation}
Then the ordinary squared regression loss
\begin{equation}
 \mathcal J_\lambda(\theta)
 :=\E_{X\sim s_\lambda,t,\epsilon}
        \|m_\theta(Y_t,t)-\widetilde A_t\|^2
 =\frac{\mathcal L_\lambda(\theta)}{1+\lambda}+C_\lambda
 \label{eq:ftfc-residual-square}
\end{equation}
has a finite constant $C_\lambda$ independent of $\theta$.
No additional endpoint weight multiplies this loss. Let
$h_t(y)=\E_q[w(X)\mid Y_t=y]$ and
$m_q(y,t)=\E_q[A_t\mid Y_t=y]$. Its unrestricted conditional optimum is
\begin{equation}
 \begin{aligned}
 m_\lambda^*(y,t)
 &=\frac{h_t m_p+(1+\lambda)m_0-m_q}{h_t+\lambda},\\
 m_\lambda^*-m_{s_\lambda}
 &=\frac{1+\lambda}{h_t+\lambda}(m_0-m_q),
 \qquad
 m_{s_\lambda}=\frac{h_t m_p+\lambda m_q}{h_t+\lambda}.
 \end{aligned}
 \label{eq:ftfc-residual-bias}
\end{equation}
Here $h_t m_p$ means $\E_q[w(X)A_t\mid Y_t=y]$, also when $h_t=0$.
In particular, an exact reference teacher recovers the native predictor
of $s_\lambda$, rather than of the unmixed target $p$.
\end{proposition}

\begin{proof}
Expanding the square in \eqref{eq:ftfc-residual-square} and using
$ds_\lambda/dq=(w+\lambda)/(1+\lambda)$ gives
\[
 C_\lambda=\frac{1}{1+\lambda}
  \E_q\!\left[\frac{(w-1)^2}{w+\lambda}
                     \|A_t-m_0\|^2\right].
\]
The remaining terms are exactly
$\mathcal L_\lambda/(1+\lambda)$. The teacher and calibrated ratios
are held fixed when differentiating. Conditional regression gives
$m_\lambda^*=\E_{s_\lambda}[\widetilde A_t\mid Y_t]$.
Since the conditional normalizer is $h_t+\lambda$,
\[
 \E_{s_\lambda}[a_\lambda(A_t-m_0)\mid Y_t]
 =\frac{h_t m_p-m_q-(h_t-1)m_0}{h_t+\lambda}.
\]
Adding $m_0$ proves the first identity; subtracting the mixture's
conditional mean proves the second.
\end{proof}

\paragraph{Symmetric mixture.}
At $\lambda=1$, draw an endpoint from $p$ or $q$ with equal probability
and use
\begin{equation}
 a_1(X)=\frac{w(X)-1}{w(X)+1}
       =\tanh\!\left(\tfrac12\log w(X)\right)\in[-1,1].
 \label{eq:ftfc-symmetric-residual}
\end{equation}
For $w=0$, the equality uses the limit $a_1=-1$.
Only a bounded scalar modifies the native label; the ratio network is absent at
inference. In general $-1/\lambda\leq a_\lambda\leq1$. The mixture strength
is a parameter of the endpoint law being fitted, not just of its estimator.
Shared conditioning-group masses are preserved, and any positive time weight
may multiply both objectives without changing the argument.

\paragraph{Posterior averaging removes component-label noise.}
One way to sample $s_\lambda$ is to draw a component variable
$B\in\{1,-1/\lambda\}$ with probabilities
$1/(1+\lambda)$ and $\lambda/(1+\lambda)$, then draw $X$ from $p$
or $q$, respectively. Bayes' rule gives
\[
 \Pr(B=1\mid X)=\frac{w(X)}{w(X)+\lambda},\qquad
 \E[B\mid X]=a_\lambda(X).
\]
The random label $m_0+B(A_t-m_0)$ has conditional expectation
$\widetilde A_t$ given the endpoint and its fresh native randomness.
For a fixed student and teacher, the squared-loss parameter gradient
is affine in this label. Consequently its soft-label gradient is the
conditional expectation of the random-component gradient. When these
gradients have finite second moments, the law of total covariance gives
\begin{equation}
 \operatorname{Cov}(G_{\rm soft})
 \preceq\operatorname{Cov}(G_{\rm component}).
 \label{eq:ftfc-residual-rb}
\end{equation}
This Rao--Blackwell bound compares the two label estimators under the
same mixture. It does not compare their variance with ordinary target
resampling, reference sampling, or other importance proposals, and it
does not imply identical Adam updates.

\paragraph{Error bound.}
Let $\widetilde s_{\lambda,t}$ and $\widetilde q_t$ be the corresponding
noisy or interpolated laws. Integrating \eqref{eq:ftfc-residual-bias}
and using
$d\widetilde s_{\lambda,t}/d\widetilde q_t=(h_t+\lambda)/(1+\lambda)$
yields
\begin{equation}
 \E_{t,\widetilde s_{\lambda,t}}
       \|m_\lambda^*-m_{s_\lambda}\|^2
 \leq\frac{1+\lambda}{\lambda}
       \E_{t,\widetilde q_t}\|m_0-m_q\|^2.
 \label{eq:ftfc-residual-teacher-bound}
\end{equation}
For $\Delta_\theta=\mathcal J_\lambda(\theta)
-\inf_m\mathcal J_\lambda(m)$, the infimum over all square-integrable
predictors, regression orthogonality and the triangle inequality give
\begin{equation}
 \|m_\theta-m_{s_\lambda}\|_{L^2(dt\,\widetilde s_{\lambda,t})}
 \leq\sqrt{\Delta_\theta}
 +\sqrt{\frac{1+\lambda}{\lambda}}\,
       \|m_0-m_q\|_{L^2(dt\,\widetilde q_t)}.
 \label{eq:ftfc-residual-error-budget}
\end{equation}
For the symmetric choice the squared-error amplification bound is two.
The bound does not estimate the actual teacher error. Even with an exact
teacher and fit, the endpoint guarantee requires the assumptions of
Corollary~\ref{thm:ftfc-overall} for $s_\lambda$; this construction is neither
an exact realization of $p$ nor a guarantee of improving the paper's utility.

\subsection{A centered regression on the reference sampler's transitions}
\label{app:ftfc-native-transition}

\paragraph{Scope.}
This optional realization uses saved transitions from the reference sampler
instead of fresh endpoint corruption. Fenchel calibration is unchanged, and
the approximate-teacher term of Appendix~\ref{app:ftfc-residual-regression}
vanishes at the population level. It is not used in the benchmark tables;
the identities alone do not imply better generation quality.

\paragraph{Utility and divergence are inputs to calibration.}
The endpoint target still solves
$\mathcal F(p)-\alpha D_f(p\|q)$ in
\eqref{eq:setup-endpoint-objective}, under the conditions of
Section~\ref{sec:method}. The construction below takes only its normalized
nonnegative ratio $w$ as input; it does not assume expected reward or KL
weights, so half-Pearson responses with zeros are allowed. The path KL is the
Gaussian transition-fitting discrepancy, distinct from the endpoint
$f$-divergence. This backend is optional; the flow and denoising realizations
in Section~\ref{ftfc:sec:realization} remain available.

Fix a conditioning value and suppress it in the notation. Let $Q$ be a
Markov reference path law with transitions
\[
 X_{t+1}=\mu_{q,t}(X_t)+\sigma_t\xi_t,\qquad
 \xi_t\sim\mathcal N(0,I),\quad \sigma_t>0,
 \qquad t=0,\ldots,T-1.
\]
Each innovation is independent of the preceding path. Write
$P=w(X_T)Q$, where $w\geq0$ and $\E_Qw=1$, and
$S_\lambda=(P+\lambda Q)/(1+\lambda)$ for $\lambda>0$.
The model $Q_\theta$ retains $Q$'s initial law and transition covariance,
and changes the transition mean to
$\mu_{q,t}(X_t)+\sigma_t r_\theta(X_t,t)$.
The correction must not observe its current innovation. Assume that the
displayed risks and relative entropies are finite.

\begin{proposition}[Centering and a conditional mixture correction]
\label{prop:ftfc-native-transition}
The centered risk
\begin{equation}
 \mathcal L^{\rm path}_\lambda(\theta)
 =\E_Q\sum_{t=0}^{T-1}\left[
   \tfrac12(w+\lambda)\|r_\theta(X_t,t)\|^2
   -(w-1)\langle\xi_t,r_\theta(X_t,t)\rangle\right]
 \label{eq:ftfc-native-centered}
\end{equation}
satisfies
\begin{equation}
 \mathcal L^{\rm path}_\lambda(\theta)
 =(1+\lambda)\left[\DKL(S_\lambda\|Q_\theta)
                         -\DKL(S_\lambda\|Q)\right].
 \label{eq:ftfc-native-kl}
\end{equation}
Define $h_t(x)=\E_Q[w(X_T)\mid X_t=x]$ and
$a_t(x)=\E_Q[w(X_T)\xi_t\mid X_t=x]$.
The unrestricted conditional minimizer is
\begin{equation}
 r_\lambda^*(x,t)=\frac{a_t(x)}{h_t(x)+\lambda}
 =\frac{h_t(x)}{h_t(x)+\lambda}\,r_P^*(x,t),\qquad
 r_P^*(x,t)=\E_P[\xi_t\mid X_t=x],
 \label{eq:ftfc-native-gate}
\end{equation}
where the second equality applies when $h_t(x)>0$.
When $h_t(x)=0$, $a_t(x)=0$ and $r_\lambda^*(x,t)=0$.
\end{proposition}

\begin{proof}
The Gaussian likelihood ratio gives
$\log(dQ/dQ_\theta)=\sum_t[\|r_\theta\|^2/2
-\langle\xi_t,r_\theta\rangle]$.
Since $r_\theta$ is adapted,
$\E_Q\langle\xi_t,r_\theta(X_t,t)\rangle=0$.
Multiplying the likelihood ratio by $w+\lambda$ and adding this zero-mean
term proves~\eqref{eq:ftfc-native-kl}. Conditional on $X_t=x$, the part
depending on $r$ is
$\tfrac12(h_t+\lambda)\|r\|^2-a_t^\top r$.
Its positive curvature gives the unique minimizer. Bayes' rule gives
$a_t=h_t\E_P[\xi_t\mid X_t]$ when $h_t>0$; nonnegativity of $w$
gives $a_t=0$ when $h_t=0$.
\end{proof}

\paragraph{A correction weighted by the current state.}
The factor $h_t/(h_t+\lambda)$ is the posterior probability of the target
component in $S_\lambda$ given $X_t$, so the ideal correction shrinks toward
the reference where target mass is small. This describes the conditional
optimum, not a bound on neural predictions or transition displacement. The
regression fits it directly and needs no estimate of $h_t$ at inference.
For another interpretation, set
$H_t(m)=\E_\xi[h_{t+1}(m+\sigma_t\xi)]$, so
$h_t(x)=H_t(\mu_{q,t}(x))$ by the Markov property. When differentiation
under this Gaussian integral is justified,
\begin{equation}
 r_\lambda^*(x,t)
 =\left.\sigma_t\nabla_m\log\bigl(H_t(m)+\lambda\bigr)
                                      \right|_{m=\mu_{q,t}(x)}.
 \label{eq:ftfc-native-mean-gradient}
\end{equation}
The derivative is with respect to the transition mean $m$, not generally
the state $x$; replacing it by a state gradient omits a Jacobian.

\paragraph{Sampling and deployment.}
Completing the square in~\eqref{eq:ftfc-native-centered} gives label
$(w-1)\xi_t/(w+\lambda)$ and curvature $w+\lambda$.
Sampling paths from $D=(P+Q)/2$ and multiplying both terms by
$dQ/dD=2/(1+w)$ preserves the risk. For $0<\lambda\leq1$, the corrected
curvature lies in $[2\lambda,2]$ and the signed linear coefficient in
$[-2,2]$. These bounds do not imply lower neural-gradient variance.
For a Gaussian diffusion step with classifier-free guidance $g$, write
\[
 \mu_q=A_tX_t+C_t[(1-g)\epsilon_q^-+g\epsilon_q^+],\qquad
 r_\theta=\frac{gC_t}{\sigma_t}
                 (\epsilon_\theta^+-\epsilon_q^+).
\]
Then the reference positive prediction cancels:
\begin{equation}
 \mu_q+\sigma_t r_\theta
 =A_tX_t+C_t[(1-g)\epsilon_q^-+g\epsilon_\theta^+].
 \label{eq:ftfc-native-cancellation}
\end{equation}
Training uses a frozen reference positive prediction, but deployment
needs only the frozen negative and learned positive branches. The branch
count alone does not establish wall-time parity with ordinary guidance.

\paragraph{Calibrating the mixture for a general objective.}
Fitting $S_\lambda$ after calibrating $P$ changes the endpoint target.
An alternative is to calibrate the desired mixture directly. Put
$\varrho=\lambda/(1+\lambda)$ and solve
\begin{equation}
 \sup_{v\geq\varrho,\,\E_qv=1}
       \{\mathcal F(vq)-\alpha\E_q f(v)\},\qquad
 w=\frac{v-\varrho}{1-\varrho}.
 \label{eq:ftfc-native-general-floor}
\end{equation}
Then $w\geq0$, $\E_qw=1$, and
$(w+\lambda)/(1+\lambda)=v$. Supplying this auxiliary ratio $w$ to
\eqref{eq:ftfc-native-centered} therefore projects toward the target
calibrated for the specified $\mathcal F$ and $f$. This is the original
endpoint objective with an explicit reference floor, not the unconstrained
problem or a replacement of its divergence by KL. Fixed condition masses
are preserved by the same affine transformation.
For a utility admitting the supporting envelope $C(g)$ from
\eqref{eq:ftfc-envelope}, replace $f_+^*$ in the dual by
\[
 f_{\varrho}^*(u)=\sup_{z\geq\varrho}\{uz-f(z)\},\qquad
 D_{f,\varrho}(g,\nu)=C(g)+\nu+
                      \alpha\E_q f_{\varrho}^*((g-\nu)/\alpha).
\]
The same Fenchel gap decomposition holds for feasible $v$; strong duality
still requires the corresponding attainment and regularity assumptions.
Where the differentiable strictly convex response is defined, it is
$v=\max\{\varrho,(f')^{-1}((g-\nu)/\alpha)\}$, with $\nu$ solved inside
the response to enforce normalization. Rescaling a non-KL response
afterward is generally invalid. For lower-CVaR plus a linear utility,
the finite-bank reward-knot argument remains valid because the reference
floor and divergence penalty do not depend on the threshold. Arbitrary
nonconcave utilities do not acquire a global guarantee from this change.

\paragraph{What the path projection does not guarantee.}
Even the unrestricted mean fit need not realize the terminal target.
The tilt generally changes the source law to
$dS_{\lambda,0}/dQ_0=(h_0+\lambda)/(1+\lambda)$, whereas the model keeps
$Q_0$. The KL chain rule therefore retains the source discrepancy
$\DKL(S_{\lambda,0}\|Q_0)$ for every $\theta$.
Furthermore, tilted conditional transitions need not be Gaussian with
covariance $\sigma_t^2I$; matching their means leaves a projection error.
Finite-bank ratios, sampled stored times, restricted neural predictors
and changed model occupancy add further errors.
Uniform time subsampling and coordinate averaging scale the population
risk by fixed positive constants; arbitrary time weights do not preserve
the path-KL identity. These limitations prevent substitution of this
projection for the exact realization assumed by
Corollary~\ref{thm:ftfc-overall}.

\section{Experimental Details}
\label{app:implementation}

\subsection{Target calibration and model fitting}
\label{app:ftfc-native-implementation}
FTFC assigns masses to cached pretrained samples and freezes them during
fitting. Direct calibration and the normalized neural-score solver in
Appendix~\ref{app:ftfc-method-details} use the same procedure.

\paragraph{Target calibration.}
Each endpoint stores a reward, utility features, and a conditioning group.
Features are normalized on the training bank, while prompt or molecular-size
probabilities are preserved. For concave utilities, the primal--dual gap
bounds finite-bank optimization error; an independent rings solver gives a
maximum gap below $2.5\times10^{-6}$. Lower-CVaR targets use global search,
since stationarity alone is insufficient. Held-out normalization, effective
sample size, and maximum density ratios assess bank coverage but do not bound
population or fitting error.

\paragraph{Constant adapter.}
For a frozen linear map $Wx$, the adapter takes the form
\begin{equation}
 Wx+U\bigl[(Vx)\odot(\gamma(c)-\gamma(0))\bigr],\qquad c=[1].
 \label{eq:ftfc-adapter}
\end{equation}
Here $U,V$ are low-rank projections and $\gamma$ is a descriptor encoder.
Initializing $U=0$ recovers the pretrained predictor. The descriptor is
constant, so each adapter represents one target distribution. We use AdamW
with zero weight decay; ranks and update budgets are given below.

\paragraph{Molecular fitting.}
We sample complete molecules by their target masses, preserving coordinates,
atom types, formal charges, and bonds. FlowMol~1 uses centered coordinates
and argmax categorical labels with coordinate and cross-entropy losses.
FlowMol3 retains its original time and modality weights, categorical masks,
fake atoms, and stochastic self-conditioning. The loss is averaged per
molecule, so size does not change its target weight. Fitting uses fresh priors
and the native objective, without reward-gradient or validity losses. At
generation time, atom counts follow the released training histogram; the
distinction from a bank restricted to reward-eligible molecules is discussed
in Appendix~\ref{app:ftfc-support}.

\paragraph{Stable Diffusion fitting.}
We cache final latents and prompt embeddings and train with fresh diffusion
noise and timesteps. For classifier-free guidance $g=7.5$, the effective
noise predictor is
\begin{equation}
 \epsilon_{\rm eff,\theta}
 =\epsilon_{\rm pre,uncond}
       +g(\epsilon_{\theta,cond}-\epsilon_{\rm pre,uncond}).
 \label{eq:ftfc-cfg}
\end{equation}
The unconditional prediction is frozen. We train on cached latents with
mean-squared error against the scheduler's noise target and the same guidance
convention as sampling; latents are not decoded and re-encoded.

\paragraph{Pretrained molecular checkpoints.}
QM9 fine-tuning starts from the released FlowMol~1 weights
\citep{dunn2024flowmol}.\footnote{\raggedright Pretrained FlowMol QM9
weights: \url{https://bits.csb.pitt.edu/files/FlowMol/trained_models_v02/qm9_gaussian/checkpoints/last.ckpt}.}
GEOM-Drugs fine-tuning starts from the released FlowMol3 v3.0 weights
\citep{dunn2025flowmol3}.\footnote{\raggedright Pretrained FlowMol3 GEOM-Drugs
weights: \url{https://bits.csb.pitt.edu/files/FlowMol/trained_models_v3/flowmol3/checkpoints/last.ckpt}.}

\subsection{QM9 benchmark protocol}
\label{app:qm9}

\paragraph{Dataset and pretrained model.}
QM9 contains 133,885 organic molecules with at most nine heavy atoms
from C, N, O, and F, plus explicit hydrogens
\citep{ramakrishnan2014qm9}. All methods start from the released FlowMol~1
checkpoint \citep{dunn2024flowmol} (Appendix~\ref{app:implementation}),
which jointly models coordinates, atom types, formal charges, and bonds.
We retain its preprocessing, Gaussian priors, feature schedules, categorical
decoding, and atom-count histogram; fine-tuning uses generated endpoints,
not the dataset's DFT labels.

\paragraph{Energy reward.}
The reward is $R(x)=-E_{\mathrm{GFN1\text{-}xTB}}(x)$ in Hartree, evaluated at
the generated, unrelaxed coordinates with the predicted total charge.
The differentiable \texttt{dxtb} oracle \citep{friede2024dxtb} supplies
coordinate gradients; atom, charge, and bond labels are decoded by argmax.
We use float32, implicit SCF differentiation, and tolerance $10^{-4}$.
Failed calculations receive $-500$ Ha; no draw is replaced, optimized, or
fragment-pruned before evaluation.

\paragraph{Baseline training.}
AM \citep{domingo2025adjoint} uses the mean-reward derivative scaled by
$1/\alpha$, with $\alpha=0.0045$, for 240 updates. FDC
\citep{de2025flow} targets the upper superquantile at $\beta=0.998$,
with no fixed terminal KL penalty. Its proximal coefficient
$\eta=0.01$ is distinct from FTFC's terminal KL coefficient.
The update counts and FDC coefficients follow Appendix E.4 of
\citet{de2025flow}; Table~\ref{tab:qm9-settings} gives the remaining
settings of our reconstruction.

\begin{table}[htbp]
  \centering
  \caption{\textbf{Fixed QM9 baseline settings.}
  FTFC settings are listed in Appendix~\ref{app:qm9-ftfc}.}
  \label{tab:qm9-settings}
  \begingroup
  \small
  \setlength{\tabcolsep}{5pt}
  \renewcommand{\arraystretch}{1.22}
  \begin{tabularx}{\linewidth}{@{}>{\raggedright\arraybackslash}p{0.38\linewidth}*{2}{>{\raggedright\arraybackslash}X}@{}}
    \toprule
    \textbf{Setting} & \textbf{AM} & \textbf{FDC} \\
    \midrule
    Objective & Mean reward, $\alpha=0.0045$ & Upper SQ, $\beta=0.998$, $\alpha=0$ \\
    Optimizer & Adam & Adam \\
    Adam moments / epsilon & $(0.9,0.999)$ / $10^{-8}$ & $(0.9,0.999)$ / $10^{-8}$ \\
    Learning rate & $10^{-4}$ & $10^{-4}$ \\
    Optimizer updates & $240$ & $10\times2=20$ \\
    Outer iterations & $1$ & $10$ \\
    Proximal coefficient $\eta$ & --- & $0.01$ \\
    Trajectories per update & $32$ & Up to $8$ tail + $24$ body \\
    Threshold bank per outer iteration & --- & $4{,}096$ fresh samples \\
    Candidate pool per inner update & --- & $4{,}096$ fresh samples \\
    Gradient-norm clipping & $10^{5}$ & $10^{5}$ \\
    Trajectory discretization & $250$ steps & $250$ steps \\
    Training sampler & Memoryless, terminal denoising & Memoryless, terminal denoising \\
    Evaluated checkpoint & Final scheduled update & Final outer iteration \\
    \bottomrule
  \end{tabularx}
  \endgroup
\end{table}

At each outer iteration, FDC estimates the 0.998-quantile from 4,096
fresh samples and holds it fixed for two inner updates. Each inner
update generates a separate pool of 4,096 trajectories and uniformly
retains up to eight tail and 24 body examples. For a stratum with
$N_h$ pool members and $n_h$ retained trajectories, each retained
example has weight $B N_h/(4096n_h)$, where $B$ is the retained batch
size. This corrects the batch mean for stratified sampling. Empty tail
strata contribute no examples. The reference model and inner optimizer
are updated and reset, respectively, at each outer iteration.

\paragraph{TFFT and tail-gradient availability.}
\label{app:qm9-tfft}
TFFT uses the AM solver for 240 updates, batch size 32, learning rate
$10^{-4}$, and $\alpha=0.0045$. Its threshold schedule is calibrated
on 10,000 independent pretrained samples. The right-tail level rises
from $\beta=0$ to $0.998$ during the first 120 updates and then remains
fixed. In the completed run, every positive-$\beta$ update had zero active
tail-reward gradients; only the initial expected-reward update had an active
signal. We report this checkpoint, which does not demonstrate sustained tail
optimization.

\subsubsection{FTFC configuration}
\label{app:qm9-ftfc}
Each of three FTFC runs uses 4,096 pretrained endpoints and upper
superquantile level $\beta=0.998$. Calibration uses rewards
$-E_{\mathrm{GFN1}}/100$ and KL coefficient $\alpha=0.01$, equivalent
to one Hartree in unscaled units. It preserves the atom-count marginal
and solves for the shared threshold and group normalizers to a
finite-bank primal--dual gap of $10^{-8}$.

With these masses fixed, we train rank-16 adapters for 80 AdamW updates:
batch size 32, microbatch size 16, learning rate $10^{-3}$, gradient-norm
clipping at 10, and float32 precision. Sampling endpoints according to
their target masses implements the weighted objective; the loss is
not weighted again. We retain the native coordinate and categorical
losses described in Appendix~\ref{app:ftfc-native-implementation}.
Evaluation uses the final adapter weights, without EMA or checkpoint
selection by test reward. The finite-bank certificate bounds target optimization error, not fitting error.

\paragraph{Sampling and repeated runs.}
Evaluation uses 250 explicit Euler steps on a uniform grid from 0.02
to 1, with 50,000 raw attempts per run in batches of 256. All methods share the evaluation atom-count list and random draws. The table
reports three FTFC runs and one run each for pretrained, AM, FDC, and TFFT,
using final scheduled checkpoints; additional AM runs share random streams and
are not independent repetitions.

\paragraph{Training time.}
\label{app:qm9-time}
Table~\ref{tab:qm9} reports wall-clock time per run on a shared RTX 4090
server, including method-specific sample acquisition, scoring, and
fitting. AM and FDC take approximately 12.0 and 35.2 hours. FTFC takes
58.4 minutes on average: $\approx3,462.0$ seconds for bank acquisition, 0.3
seconds for calibration, and 40.5 seconds for fitting. The acquisition
cost is charged even when the bank is reused. TFFT takes approximately
$15$ hours, including its $10,000$-sample threshold calibration.
Backbone pretraining, evaluation, and development runs are excluded. These
figures compare the stated budgets under shared-server load, not time to a
matched reward.

\subsubsection{Evaluation metrics}
\label{app:qm9-metrics}
Table~\ref{tab:qm9} uses all raw attempts, including the failure penalty;
all reported attempts have successful energy calculations. For ascending rewards
$R_{(1)}\leq\cdots\leq R_{(N)}$, the exact empirical superquantile is
\begin{equation}
  \widehat{\mathrm{SQ}}_\beta
  =\frac{\sum_{j=1}^{k}R_{(N-j+1)}+\delta R_{(N-k)}}{m},
  \qquad m=(1-\beta)N,\quad k=\lfloor m\rfloor,\quad\delta=m-k,
  \label{eq:qm9-fractional-sq}
\end{equation}
with the boundary term omitted when $\delta=0$. At $N=50{,}000$, the
SQ$_{0.998}$ and upper-10\% columns average the best 100 and 5,000
rewards, respectively.

Graph validity requires RDKit sanitization, a single connected
component, and successful canonical-SMILES conversion. It is reported
as a percentage of all attempts. Synthetic accessibility (SA)
\citep{ertl2009synthetic} is averaged over valid molecules with a
successful score, after removing explicit hydrogens; lower is better.
FTFC entries report the mean and sample standard deviation across
three runs. Baseline entries are individual-run measurements.

\paragraph{Topology-based validity reward.}
\label{app:qm9-validity}
We also evaluate $R_{\mathrm{valid}}\in[0,1]$, adapted from
\citet[Eq.~13]{kotani2026atomic}, using
Eq.~\eqref{eq:geomdrugs-validity-reward} and the GFN2-xTB protocol in
Appendix~\ref{app:geomdrugs-validity}. Generated hydrogens and formal
charges are retained. For molecules without hydrogens, a single-point
energy replaces the hydrogen-relaxation reference. Every raw draw
contributes, with zero assigned to structural or numerical failures;
RDKit is not a prefilter. This metric is evaluation-only and differs
from both graph validity and the GFN1 training reward.

The topology evaluation regenerates the original sampling populations from
the final checkpoints. Canonical SMILES, graph-validity flags, and available
geometric checks agree with the original evaluations, but missing coordinates
prevent exact identity checks for every geometry.

\paragraph{Reproduction scope.}
All comparisons use this protocol and a shared checkpoint. FDC's public
description omits some molecular sampling conventions, so our reconstruction
has not been verified against the authors' implementation or published results.

\subsection{GEOM-Drugs benchmark protocol}
\label{app:geomdrugs}

\paragraph{Pretrained model.}
We reconstruct the molecular task of \citet{wang2026efficient} using
the released FlowMol3 v3.0 checkpoint \citep{dunn2025flowmol3},
pretrained on GEOM-Drugs with explicit hydrogens and Kekulized bonds
\citep{axelrod2022geom}. All methods retain its atom-count histogram,
coordinate dynamics, categorical continuous-time Markov chains
(CTMCs), self-conditioning, and fake-atom handling. The results are
computed with this shared implementation rather than taken from
published tables.

\paragraph{Energy reward.}
Tables report $R(x)=-E_{\mathrm{GFN1}}(x)$ in atomic units; training
uses $6R(x)$ and $\alpha=1$, following \citet{wang2026efficient}.
We compute energies and coordinate derivatives with \texttt{dxtb}
\citep{friede2024dxtb} in float64, using the predicted total formal
charge and excluding fake atoms. Calculations allow 300 SCF iterations
at tolerance $10^{-8}$ and one retry. Invalid molecules or failed
calculations contribute no reward derivative and are ineligible for
threshold estimation. They remain in the evaluation population.
We apply neither energy clipping nor terminal-gradient clipping.

\paragraph{Baseline training.}
Table~\ref{tab:geomdrugs-settings} summarizes the settings. EXP-FT
uses Adjoint Matching \citep{domingo2025adjoint} with expected reward;
R-TFFT uses the same solver with a hinge reward. FDC
\citep{de2025flow} uses three successive 120-update stages and two
additional threshold-sampling stages. Evaluation uses the final
scheduled checkpoints. The methods therefore differ in total compute.

\begin{table}[htbp]
  \centering
  \caption{\textbf{Fixed molecular benchmark settings.}
  Method-specific choices and evaluation definitions are detailed in
  Appendix~\ref{app:geomdrugs}.}
  \label{tab:geomdrugs-settings}
  \small
  \setlength{\tabcolsep}{5pt}
  \renewcommand{\arraystretch}{1.12}
  \begin{tabularx}{\linewidth}{@{}lX@{}}
    \toprule
    \textbf{Setting} & \textbf{Value} \\
    \midrule
    Runs per method & $3$ \\
    Reward / KL coefficient & $6(-E_{\mathrm{GFN1}})$, $\alpha=1$ \\
    Target right-tail level & $\beta=0.9$ \\
    Optimizer & AdamW; learning rate $10^{-4}$; weight decay $0$ \\
    Updates & EXP-FT / R-TFFT: $120$; FDC: $3\times120$ \\
    Effective batch / microbatch & $8$ / at most $2$ molecules \\
    Training time grid & $100$ uniform transitions on $[0,1]$ \\
    Adjoint interval / loss records & $t\geq0.5$; $10$ selected records \\
    Gradient clipping & None (optimizer and terminal reward gradients) \\
    R-TFFT curriculum & $\beta_j=0.9j/119$, $j=0,\ldots,119$ \\
    Shared prior corpus & $10{,}000$ draws; $9{,}988$ eligible \\
    Moving FDC corpora & $10{,}000$ draws before each of stages $2$ and $3$ \\
    FDC stage coefficients / score time & $(1.2,3.1,5.0)$; $t=0.98$ \\
    Evaluation per run & $2{,}000$ draws; batch size $16$ \\
    Evaluation / corpus sampler & Native Euler/CTMC; $250$ time-grid points \\
    CTMC sampling parameters & Stochasticity $30$; high-confidence threshold $0.9$ \\
    Energy precision / SCF & Float64; $300$ iterations; tolerance $10^{-8}$ \\
    \bottomrule
  \end{tabularx}
\end{table}

Training uses 100 uniform memoryless transitions with coordinate noise
centered per molecule. The adjoint runs over $t\geq0.5$; each loss uses the
five records nearest the endpoint and five sampled from the remaining
eligible records. Self-conditioning contexts are detached during the adjoint.
EXP-FT and R-TFFT use residual weights $(3,0.4,1,2)$ for coordinates,
atom types, charges, and bonds; FDC uses only the coordinate residual (weight
1).

\paragraph{FDC thresholds and regularization.}
The first stage uses the empirical 0.9-quantile of the shared pretrained
corpus; each later stage refreshes it from 10,000 samples of the preceding
model. With $\eta=(1.2,3.1,5.0)$, the reward derivative is the hinge gradient
scaled by $1/[\eta_k(1-\beta)]$, and the density regularizer is
$-\alpha/\eta_k$ times the estimated current-to-pretrained coordinate-score
difference. Scores use endpoint predictions at $t=0.98$ without
self-conditioning. The pretrained anchor is fixed; each stage freezes its
reference model and carries over AdamW moments. We have not verified these
settings against the authors' molecular implementation.

\paragraph{R-TFFT thresholds.}
The shared corpus contains 10,000 pretrained molecules; $M=9{,}988$ have
valid graphs and converged energies. The curriculum is
$\beta_j=0.9j/119$ for $j=0,\ldots,119$. For $\beta_j>0$, solve
\begin{equation}
  \tau_j \in \argmin_\tau
  \left\{\tau+a\log\left[
    \frac{1}{M}\sum_{i=1}^{M}
    \exp\left(\frac{[R_i-\tau]_+}{a(1-\beta_j)}\right)
  \right]\right\},
  \qquad a=\frac{\alpha}{6}=\frac{1}{6}.
  \label{eq:geomdrugs-dual}
\end{equation}
Use bounded scalar minimization with tolerance $10^{-7}$ on
$[R_{\min}-2s,R_{\max}+2s]$, where
$s=\max(R_{\max}-R_{\min},\operatorname{std}(R),1)$. At $\beta_0=0$,
$\tau_0=R_{\min}-20s$ recovers the expected-reward derivative. The training
pseudo-reward is $6[R(x)-\tau_j]_+/(1-\beta_j)$.

\paragraph{Tail-gradient availability.}
\label{app:geomdrugs-tail}
For positive curriculum values, thresholds range from 176.8775 to 179.0905,
but only one eligible pretrained reward exceeds them. No eligible training
molecule exceeds its threshold in any R-TFFT run, so positive-$\beta$ updates
receive no active hinge-reward gradient. The initial expected-reward update
and later reference matching can still change the model. We report all three
runs; this is a limitation of our finite-sample reconstruction, not of TFFT
in general.

\subsubsection{Evaluation samples and metrics}
\label{app:geomdrugs-metrics}
Each method is evaluated in three runs with 2,000 molecules per run,
using the native Euler/CTMC sampler with 250 time-grid points and
batch size 16. Evaluation uses common random draws across methods.
Failed or repeated molecules are retained without replacement or
deduplication.

\paragraph{Energy and tail reward.}
Mean reward averages the unscaled negative GFN1 energy over molecules
that pass RDKit sanitization and have a finite, converged energy.
For the $N_s$ eligible rewards in run $s$, ordered increasingly, the
right-tail statistic is
\begin{equation}
  \widehat{\mathrm{R\mbox{-}CVaR}}_{0.9,s}
  =\frac{1}{m_s}\sum_{i=N_s-m_s+1}^{N_s}R_{(i)},
  \qquad m_s=\lceil0.1N_s\rceil.
  \label{eq:geomdrugs-metrics}
\end{equation}
This averages the best-scoring tenth, with equal weights and no
fractional boundary correction. Both energy metrics condition on
successful scoring. Since they use total energies of molecules with
varying size and composition, higher reward alone does not establish
better geometry or synthesizability.

\paragraph{Graph validity and synthetic accessibility.}
Validity is the percentage of all draws passing RDKit sanitization.
It does not require connectivity, energy convergence, or a stable
three-dimensional geometry. SA \citep{ertl2009synthetic} is averaged
over sanitized molecules with finite scores, regardless of energy
convergence; lower is better. The evaluation retains explicit hydrogens
when computing SA. This differs from the QM9 convention and must be
accounted for when comparing SA across benchmarks.

\paragraph{Aggregation.}
Tables report the mean and sample standard deviation across three runs,
computed from each run's metric. Validity uncertainty is in percentage
points. The runs share one pretrained corpus, so this variation does
not include uncertainty from acquiring a new corpus. Energy and SA
use the eligible populations above; the topology reward below uses
all generated molecules.

\paragraph{Qualitative examples.}
\label{app:geomdrugs-random-samples}
Figure~\ref{fig:geomdrugs-random-samples} shows four molecules per method.
The baselines use the first four draws from a separate sampling run.
FTFC uses an earlier checkpoint than the quantitative table: three
panels are unselected draws, while Draw~2 is the first new valid molecule
with converged reward $R\geq80$. The first candidate met this threshold.
This selected panel is excluded from quantitative evaluation and
cannot be used to compare overall performance.

All methods use the same sampling settings as the benchmark. The
annotations give the unscaled GFN1 rewards at the generated geometries.
Rendering applies only rigid alignment, hides hydrogens, and fits
each molecule to its panel; it does not optimize coordinates. Panel
sizes therefore do not represent a common spatial scale.

\subsubsection{Topology-based validity reward}
\label{app:geomdrugs-validity}
We adapt the validity reward of \citet[Eq.~13]{kotani2026atomic} to
measure whether relaxation preserves a generated molecule's heavy-atom
connectivity and how much energy it releases. All 30,000 benchmark
draws are scored, including RDKit-rejected graphs and failed GFN1
calculations. This metric is used only for evaluation.

\paragraph{Definition.}
Let $A_i$ indicate that molecule $i$ passes the connectivity, collision,
relaxation, and hydrogen-integrity checks below. We call this event
xTB topology preservation (XTP). Let $h_i$ be its number of heavy atoms,
$E_i^{\mathrm H}$ the energy after optimizing hydrogens with heavy atoms
fixed, and $E_i^{\mathrm F}$ the energy after unconstrained relaxation.
With energies in kcal/mol, the reward is
\begin{equation}
  v_i =
  \begin{cases}
    0, & A_i=0,\\
    0.60+0.40\exp(-\varepsilon_i/T_\varepsilon), & A_i=1,
  \end{cases}
  \qquad
  \varepsilon_i=\frac{|E_i^{\mathrm H}-E_i^{\mathrm F}|}{h_i},
  \quad T_\varepsilon=2.00.
  \label{eq:geomdrugs-validity-reward}
\end{equation}
The strain $\varepsilon_i$ and scale $T_\varepsilon$ are in kcal/mol
per heavy atom. The absolute energy gap follows the released reward
implementation and agrees with the signed relaxation decrease for
all passing samples here. Passing molecules receive a base reward of
0.60, with up to 0.40 more for low strain; any structural or numerical
failure receives zero.

\paragraph{Relaxation protocol.}
We retain generated hydrogens and formal charges, remove fake atoms, and
choose the minimum spin consistent with electron-count parity. The reference
heavy-atom graph uses predicted bonds without bond orders; RDKit sanitization
is not a prefilter. After connectivity and the released ADT collision checks,
GFN2-xTB relaxation proceeds in three stages:
\begin{enumerate}[leftmargin=*,itemsep=1pt,topsep=3pt]
  \item Optimize hydrogens with all heavy atoms fixed, using L-BFGS
  for at most 500 cycles.
  \item Relax all atoms while restraining the predicted heavy-atom
  bond distances to their values after hydrogen relaxation, with
  force constant 0.5.
  \item Remove all restraints and perform unconstrained relaxation.
\end{enumerate}
Every stage must converge with finite coordinates and energy. We use xTB
6.7.1, accuracy 1.0, at most 250 SCC iterations, and standard geometry
optimization tolerances. Time limits are 300 seconds for hydrogen relaxation
and 180 seconds for each later stage; timeouts and numerical failures score zero.

In the final geometry, atoms $i,j$ are connected when
$\|x_i-x_j\|<1.3[r_{\mathrm{cov}}(Z_i)+r_{\mathrm{cov}}(Z_j)]$,
using Cordero covalent radii. The resulting heavy-atom edge set must
match the predicted graph exactly. Hydrogens farther than 1.6\,\AA{}
from every heavy atom are detached. An even number of detached
hydrogens is removed and relaxation restarts, up to three times;
an odd number causes rejection. These rules—native hydrogens, rejection of odd detachments, and convergence
at every stage—differ from the authors' full generation and scoring pipeline.

\paragraph{Aggregation and interpretation.}
For $n=2{,}000$ molecules in run $s$, define
\begin{equation}
  R_{\mathrm{valid},s}=\frac{1}{n}\sum_{i=1}^{n}v_{s,i},
  \qquad
  p_s=\frac{1}{n}\sum_{i=1}^{n}A_{s,i},
  \qquad
  q_s=\frac{\sum_{i:A_{s,i}=1}v_{s,i}}{\sum_i A_{s,i}}.
  \label{eq:geomdrugs-validity-aggregation}
\end{equation}
Here $p_s$ is the XTP pass rate and $q_s$ the conditional mean reward.
When $p_s>0$,
\begin{equation}
  R_{\mathrm{valid},s}=p_s q_s,
  \qquad
  0.60p_s\leq R_{\mathrm{valid},s}\leq p_s\leq1.
  \label{eq:geomdrugs-validity-decomposition}
\end{equation}
Failures remain in the denominator, so $R_{\mathrm{valid}}$ combines topology
preservation with the strain of passing molecules; it is not a validity
percentage. It also differs from the GFN1 training reward. Local connectivity
preservation does not establish a global energy minimum or agreement with a
reference conformer.

Table~\ref{tab:geomdrugs-validity-diagnostics} separates these effects. FTFC has
the highest observed pass rate, while pretrained has the lowest relaxation
strain. R-TFFT passes more often than EXP-FT and FDC but has greater strain
among passing molecules; graph validity or XTP rate alone is therefore
insufficient to characterize the combined reward.

\begin{table}[htbp]
  \centering
  \caption{\textbf{Decomposition of the additional validity reward.}
  These diagnostics cover the same 30,000 draws as Table~\ref{tab:geomdrugs}.
  XTP is the fraction of all draws that pass the complete protocol.
  The remaining columns condition on XTP: mean validity reward and
  mean relaxation strain in kcal/mol per heavy atom. Each entry is
  the mean \(\pm\) sample standard deviation of three runs.
  Bold values mark the best means before rounding: higher XTP and reward,
  lower strain.}
  \label{tab:geomdrugs-validity-diagnostics}
  \begingroup
  \resultstablestyle
  \setlength{\tabcolsep}{8pt}
  \begin{tabular}{@{}lccc@{}}
    \toprule
    \textbf{Method} & \textbf{XTP (\%)} & \textbf{Reward given XTP} & \textbf{Strain given XTP} \\
    \midrule
    Pretrained & $98.52\uncertainty{0.20}$ & $\bestvalue{0.97}\uncertainty{0.00}$ & $\bestvalue{0.18}\uncertainty{0.00}$ \\
    EXP-FT (AM) & $97.68\uncertainty{0.12}$ & $0.95\uncertainty{0.00}$ & $0.27\uncertainty{0.02}$ \\
    FDC ($K=3$) & $97.68\uncertainty{1.08}$ & $0.95\uncertainty{0.01}$ & $0.27\uncertainty{0.04}$ \\
    R-TFFT & $98.23\uncertainty{0.26}$ & $0.94\uncertainty{0.02}$ & $0.32\uncertainty{0.11}$ \\
    \midrule

    \ourmethod{FTFC (ours)} & $\bestvalue{98.67}\uncertainty{0.08}$ & $0.97\uncertainty{0.00}$ & $0.19\uncertainty{0.00}$ \\
    \bottomrule
  \end{tabular}
  \endgroup
\end{table}

\subsubsection{FTFC configuration and training cost}
\label{app:geomdrugs-training-time}

\paragraph{FTFC training.}
FTFC solves the upper-CVaR dual with $\beta=0.9$, KL coefficient
$\alpha=1$, and reward scale six on the 9,988 eligible pretrained
endpoints. Normalization preserves the empirical masses of 92 molecular
size groups. The fixed target is fitted with a rank-16 adapter for
120 updates: batch size eight, microbatch size one, learning rate
$10^{-4}$, float32 precision, and gradient-norm clipping at ten.
No new reward queries are needed during fitting. Evaluation uses
the final adapter weights.

\paragraph{Cost with a cached corpus.}
Table~\ref{tab:geomdrugs-training-time} reports active GPU-allocation
time per run on an RTX 3090 Ti, including CPU work within training.
It includes setup, calibration, fitting, and method-specific online
sampling and scoring. The shared pretrained corpus is already available;
evaluation, backbone pretraining, and interruption downtime are excluded.
FTFC calibration takes 0.246 seconds and fitting 132.52 seconds on
average, or 2.36 minutes including setup and export. FDC additionally
requires two 10,000-molecule threshold refreshes. These timings compare
the stated training budgets, rather than time to reach the same reward; the
R-TFFT comparison also has the tail-gradient limitation described in
Appendix~\ref{app:geomdrugs-tail}.

\begin{table}[htbp]
  \centering
  \caption{\textbf{Training time on GEOM-Drugs.}
  Active minutes per run on one RTX 3090 Ti, with the shared pretrained
  corpus available. Values are mean \(\pm\) sample standard deviation
  over three runs. Bold values mark the lowest mean times among
  fine-tuned methods, including ties; updates are prescribed budgets.}
  \label{tab:geomdrugs-training-time}
  \begingroup
  \resultstablestyle
  \setlength{\tabcolsep}{4pt}
  \begin{tabularx}{\linewidth}{@{}l r *{3}{>{\centering\arraybackslash}X}@{}}
    \toprule
    \textbf{Method} & \textbf{Updates}
    & \textbf{Fitting + setup} & \textbf{Threshold refresh}
    & \textbf{Total training} \\
    \midrule

    Pretrained & $0$ & $0.00$ & $0.00$ & $0.00$ \\
    EXP-FT (AM) & $120$
    & $129.40\uncertainty{0.89}$ & $\bestvalue{0.00}$
    & $129.40\uncertainty{0.89}$ \\
    FDC ($K=3$) & $360$
    & $388.11\uncertainty{2.07}$
    & $414.65\uncertainty{11.03}$
    & $802.75\uncertainty{13.02}$ \\
    R-TFFT & $120$
    & $125.83\uncertainty{0.36}$ & $\bestvalue{0.00}$
    & $125.83\uncertainty{0.36}$ \\
    \midrule

    \ourmethod{FTFC (ours)} & $120$
    & $\bestvalue{2.36}\uncertainty{0.08}$ & $\bestvalue{0.00}$
    & $\bestvalue{2.36}\uncertainty{0.08}$ \\
    \bottomrule
  \end{tabularx}
  \par\vspace{4pt}
  \begin{minipage}{\linewidth}
    \footnotesize\color{black!65}
    Evaluation, original backbone pretraining, and documented interruption intervals
    are excluded. FDC's two required 10,000-molecule threshold refreshes
    are included. The common offline corpus is accounted for separately
    in the text. Total uncertainty is computed from per-run totals;
    rounded component means need not sum exactly to the rounded total.
  \end{minipage}
  \endgroup
\end{table}


\begin{figure}[htbp]
  \centering
  \includegraphics[width=\linewidth]{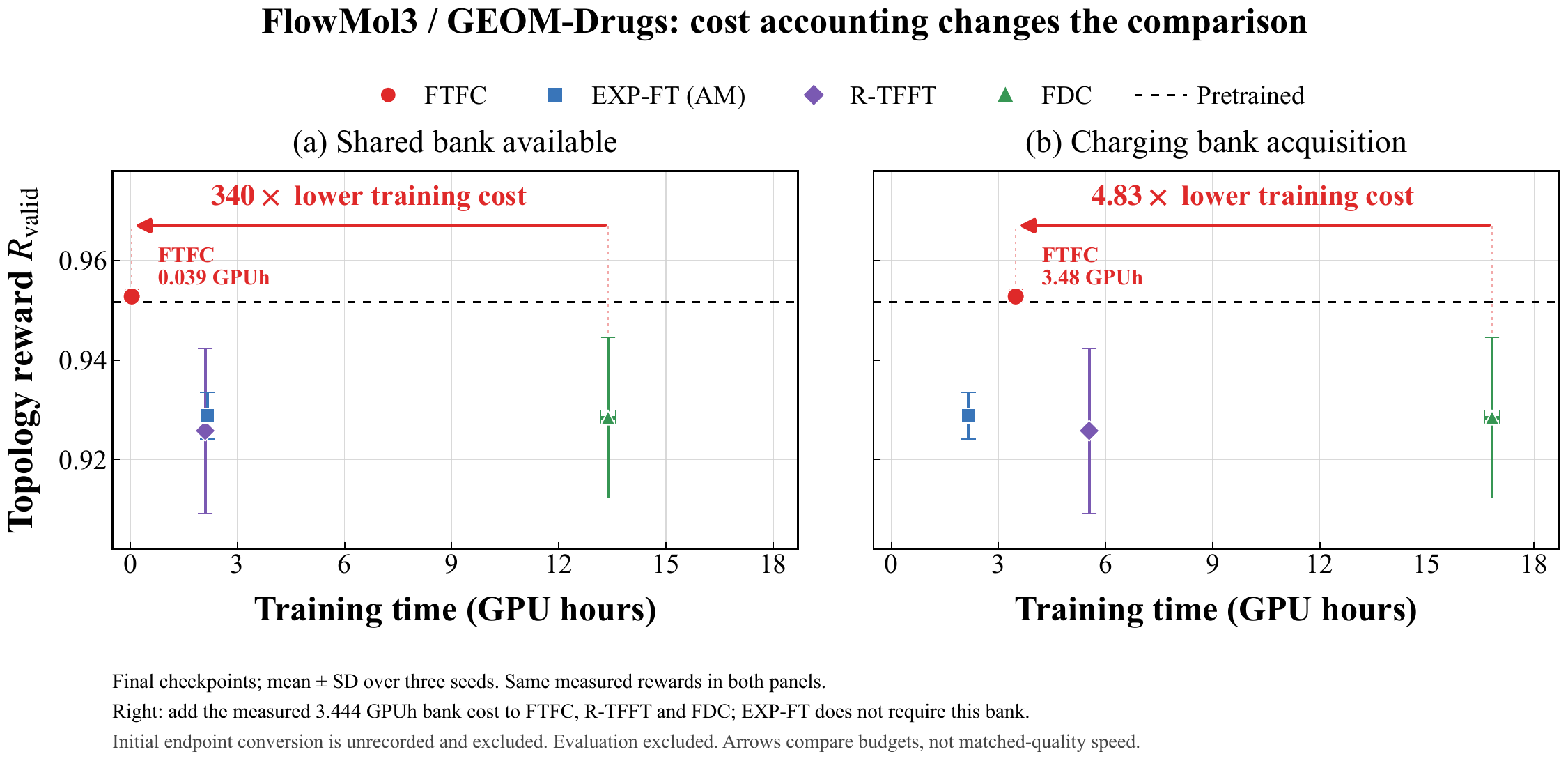}
  \caption{\textbf{Effect of corpus acquisition on GEOM-Drugs training cost.}
  Both panels use the same final topology rewards, averaged over three
  runs with 2,000 generated molecules per run; error bars
  show sample standard deviations. (a) The shared corpus is available.
  (b) Its measured acquisition cost, 3.444 GPU-hours, is charged once
  to each method that requires it. Initial endpoint conversion was
  not timed and is excluded, as is evaluation. Arrows compare recorded
  budgets rather than time to a matched reward. The horizontal dashed
  line shows the pretrained topology reward.}
  \label{fig:geomdrugs-training-cost}
\end{figure}


\subsection{Stable Diffusion benchmark protocol}
\label{app:stable-diffusion}

\paragraph{Model and prompts.}
All methods start from Stable Diffusion v1.5. Baselines use the released
TFFT implementation,\footnote{\url{https://github.com/zifan-5528/TFFT}}
which updates the full U-Net while freezing the VAE, text encoder, reward
model, and pretrained reference. FTFC uses the adapter procedure in
Appendix~\ref{app:ftfc-native-implementation}.

The training list has 10,000 records and 9,803 distinct prompts; evaluation
uses 100 prompts, two of which also occur in training, including the astronaut
prompt in Figure~\ref{fig:stable-diffusion-qualitative}. Baselines use their
final scheduled checkpoints. FTFC uses the evaluated checkpoint with the
highest lower-tail CVaR, so its result includes evaluation-set selection.

\paragraph{Baseline objectives and budgets.}
For raw ImageReward $r(x,c)$, EXP-FT uses Adjoint Matching
\citep{domingo2025adjoint} with the identity reward. L-TFFT and our
two-stage FDC reconstruction use a lower-tail hinge:
\begin{equation}
  \phi_{\mathrm{EXP}}(r)=r,
  \qquad
  \phi_{\tau}(r)=-\frac{[\tau-r]_+}{\beta},
  \qquad \beta=0.2,\quad [a]_+=\max(a,0).
  \label{eq:sd-reward-transform}
\end{equation}
Table~\ref{tab:sd-method-settings} gives thresholds and budgets. Thresholds
are fixed within each stage. FDC's second stage continues from the first with
a fresh optimizer and warmup. These settings reconstruct the baselines
\citep{de2025flow,wang2026efficient}; exact agreement with the authors'
experiments has not been established.

\begin{table}[!htbp]
  \centering
  \small
  \setlength{\tabcolsep}{4pt}
  \caption{Stable Diffusion method budgets and training thresholds.
  All trainable baselines are assigned 1,600 optimizer updates in total.}
  \label{tab:sd-method-settings}
  \begin{tabularx}{\linewidth}{@{}l*{3}{>{\centering\arraybackslash}X}@{}}
    \toprule
    \textbf{Method} & \textbf{Epochs} & \shortstack{\textbf{Optimizer}\\\textbf{updates}} & \shortstack{\textbf{Raw hinge}\\\textbf{threshold $\tau$}} \\
    \midrule
    Pretrained~\methodcite{rombach2022latent} & $0$ & $0$ & --- \\
    EXP-FT~\methodcite{domingo2025adjoint} & $20$ & $1{,}600$ & Identity reward \\
    L-TFFT~\methodcite{wang2026efficient} & $20$ & $1{,}600$ & $0.700759$ \\
    FDC, stage $1$~\methodcite{de2025flow} & $10$ & $800$ & $-0.49$ \\
    FDC, stage $2$~\methodcite{de2025flow} & $10$ & $800$ & $-0.09$ \\
    \bottomrule
  \end{tabularx}
\end{table}

\paragraph{Gradient estimation and optimization.}
The terminal adjoint is $-100g$. For each endpoint, we add 20 Gaussian
perturbations with standard deviation 0.02, decode and score them, and
differentiate the identity or hinge reward. Norms are clipped at the per-image
85th percentile before averaging, so the hinge acts before smoothing and
clipping. The control penalty remains relative to the frozen pretrained U-Net.
Table~\ref{tab:sd-training-settings} gives the remaining settings: each epoch
has 80 updates with 125 replay records per update; each loss minibatch samples
two denoising positions from the first 60\% of the trajectory and two from the
remainder. Replay losses are not importance-weighted. Guidance is used only
for image sampling, not adjoint or control computations.

\begin{table}[!htbp]
  \centering
  \small
  \renewcommand{\arraystretch}{1.08}
  \caption{Shared Stable Diffusion baseline training settings. Reward-gradient
  clipping during smoothing is distinct from optimizer-gradient clipping.}
  \label{tab:sd-training-settings}
  \begin{tabularx}{\linewidth}{@{}lX@{}}
    \toprule
    \textbf{Setting} & \textbf{Value} \\
    \midrule
    Optimizer & AdamW; learning rate $3\times10^{-6}$ \\
    Adam moments, epsilon & $(0.9,0.95)$, $10^{-8}$ \\
    Weight decay & $0$ \\
    Learning-rate schedule & Linear warmup for $20$ updates, then constant \\
    Minibatch / accumulation & $5$ records / $25$ minibatches \\
    Replay buffer / passes & $100$ trajectories / $10$ passes \\
    Trajectory / loss positions & $50$ denoising steps / $4$ sampled positions \\
    Training reward & ImageReward-v1.0; multiplier $100$ \\
    Tail objective & Lower tail, $\beta=0.2$, for FDC and L-TFFT \\
    Gradient smoothing & $20$ latent perturbations, Gaussian std.\ $0.02$ \\
    Smoothing gradient clipping & Per-image $0.85$-quantile of perturbation norms \\
    Other clipping & Optimizer-gradient and per-sample loss clipping disabled \\
    Guidance / prompt dropout & Sampling CFG $7.5$ / dropout $0.2$ \\
    Adjoint / control CFG & Both disabled \\
    Precision & Float32; TF32 permitted \\
    Validation frequency & Every $0.1$ epoch on the $100$ evaluation prompts \\
    Final model selection & Last checkpoint after the scheduled updates \\
    Hardware per worker & One NVIDIA A100-SXM4 with $80$\,GB memory \\
    \bottomrule
  \end{tabularx}
\end{table}

\paragraph{Sampling.}
Training and evaluation use the released stochastic DDIM scheduler
with 50 steps, $\eta=1$, trailing timestep spacing, and noise-schedule
endpoints 0.002 and 0.009. The terminal cumulative alpha is adjusted
to $(1+\bar\alpha_{t_{\mathrm{last}}})/2$. Each method generates ten
$512\times512$ images for each of 100 prompts, using guidance 7.5 and
common random draws across methods. All 1,000 images enter quantitative
evaluation without selection. Scorers use their released preprocessing
and unperturbed images.

\paragraph{ImageReward and lower-tail CVaR.}
ImageReward-v1.0 \citep{xu2023imagereward} predicts preference for a
prompt--image pair. Let $r_{pi}$ be its raw score, $P=100$ the number
of prompts, and $n=10$ the number of images per prompt. The mean is
\begin{equation}
  \widehat\mu_{\mathrm{IR}}
  =\frac{1}{P}\sum_{p=0}^{P-1}\bar r_p,
  \qquad \bar r_p=\frac{1}{n}\sum_{i=1}^{n}r_{pi}.
  \label{eq:sd-mean-reward}
\end{equation}
Sorting all $M=Pn$ scores jointly gives the lower-tail statistic
\begin{equation}
  \widehat{\mathrm{L\mbox{-}CVaR}}_{0.2}
  =\frac{1}{K}\sum_{j=1}^{K}r_{(j)},
  \qquad K=\lceil0.2M\rceil=200.
  \label{eq:sd-left-cvar}
\end{equation}
This is a global tail across prompts, not an average of per-prompt
tails. Evaluation uses raw scores and an empirical tail boundary,
independently of the fixed training thresholds.

\paragraph{CLIP and HPSv2.1.}
CLIP ViT-L/14 \citep{radford2021clip} measures text--image compatibility
through the cosine similarity of image and text features:
\begin{equation}
  C(x,c)=\frac{f(x)^\top h(c)}{\|f(x)\|_2\|h(c)\|_2}.
  \label{eq:sd-clip}
\end{equation}
HPSv2.1 \citep{wu2023hpsv2} supplies a separate learned preference
score. Both are reported without rescaling, averaged within prompts
and then across prompts. Neither is used for training.

\paragraph{DreamSim diversity.}
For unit-normalized embeddings $z_{pi}$ from the default pretrained
DreamSim model \citep{fu2023dreamsim}, we compute
\begin{equation}
  D_p=\frac{1}{2n(n-1)}\sum_{i=1}^{n}\sum_{j=1}^{n}
      \|z_{pi}-z_{pj}\|_2^2,
  \qquad \widehat D=\frac{1}{P}\sum_{p=0}^{P-1}D_p.
  \label{eq:sd-dreamsim}
\end{equation}
This is within-prompt perceptual variance, used only for evaluation; higher
diversity need not imply higher image quality or prompt alignment.

\paragraph{Uncertainty.}
The table reports one evaluation per method. For ImageReward mean,
CLIP, and HPS, let $a_p$ be the ten-image prompt mean; for DreamSim,
let $a_p=D_p$. The reported error term is
\begin{equation}
  \mathrm{SE}_{\mathrm{prompt}}
  =\frac{\sqrt{P^{-1}\sum_{p=0}^{P-1}(a_p-\bar a)^2}}{\sqrt P},
  \qquad \bar a=P^{-1}\sum_p a_p.
  \label{eq:sd-prompt-se}
\end{equation}
For lower-tail CVaR, it is instead
\begin{equation}
  \mathrm{SE}_{\mathrm{tail}}
  =\frac{\sqrt{K^{-1}\sum_{j=1}^{K}
    (r_{(j)}-\widehat{\mathrm{L\mbox{-}CVaR}}_{0.2})^2}}{\sqrt K},
  \qquad K=200.
  \label{eq:sd-tail-se}
\end{equation}
The tail term describes dispersion within the selected tail; it omits boundary
uncertainty and dependence between images of one prompt. Neither term measures
training-run variability or gives a confidence interval for method differences.

\paragraph{Training cost.}
Baseline timings on one A100 80\,GB GPU include active training, validation,
checkpoint writing, and repeated partial work after interruptions. They exclude
downtime, external startup, final evaluation, backbone pretraining, and
development runs.
The recorded totals are 82.28 hours for EXP-FT, 82.41 for FDC, and
85.77 for L-TFFT.

The selected FTFC checkpoint combines 800 initial updates and a
400-update refinement. Fitting takes 4.045 GPU-hours with a cached
bank. Generating and scoring its 80,000-image bank adds 43.711
GPU-hours, giving 47.756 GPU-hours (1.99 GPU-days) in total. FTFC's
timers for these stages exclude CPU calibration, startup, checkpoint writing,
validation, and exploratory or discarded updates. Because timing scopes and attained rewards differ, these costs do not establish
a speedup to matched quality.

\paragraph{Qualitative selection.}
Figure~\ref{fig:stable-diffusion-qualitative} shows ImageReward ranks
1, 5, and 10 among ten images per method for ``\textit{footage of an astronaut
in a tropical beach}.'' This prompt was chosen after inspecting
pretrained outputs and also occurs in training. Pretrained, EXP-FT,
L-TFFT, and FTFC reuse their evaluation images. The FDC batch was
resampled twice after visual inspection; its displayed ranks use the
last batch. This selection affects only the illustration, not the
quantitative benchmark. No images are retouched.

\subsubsection{L-TFFT: fixed-threshold lower-tail optimization}
\label{app:sd-ltfft}
L-TFFT applies the lower-tail transformation of
\citet{wang2026efficient} with $\beta=0.2$ and fixed threshold
$\tau=0.700759$:
$\phi_\tau(r)=-5[0.700759-r]_+$.
For a terminal latent $z$ and decoded reward $r(z,c)$, its unsmoothed
derivative is
\begin{equation}
  \nabla_z\phi_\tau(r(z,c))
  =5\,\mathbf{1}\{r(z,c)<\tau\}\,\nabla_z r(z,c),
  \label{eq:sd-ltfft-gradient}
\end{equation}
away from the threshold. Low-reward samples receive a direct reward gradient, while the KL control
penalty remains active for all samples. The threshold is fixed, so the fraction
below it need not remain 20\%.

The hinge is applied to each latent perturbation before clipping and averaging,
so an image above the threshold can still receive a gradient when a perturbation
falls below it. Training uses the shared Adjoint Matching settings and 1,600
updates in Tables~\ref{tab:sd-method-settings}--\ref{tab:sd-training-settings}.

L-TFFT improves mean ImageReward and lower-tail CVaR over pretrained, but not
over EXP-FT in this evaluation. With one run per method, this difference does
not establish equivalence or statistical superiority.

\clearpage
\section{Additional Results}
\label{app:additional-results}

law's preference for the left ring. FTFC has the smallest target distance at
$L=64$ (Figure~\ref{fig:toy2d-left-cvar}), reproducing both modes and
their unequal masses more closely than ITM or FDC.

\begin{figure}[!ht]
  \centering
  \includegraphics[width=\linewidth]{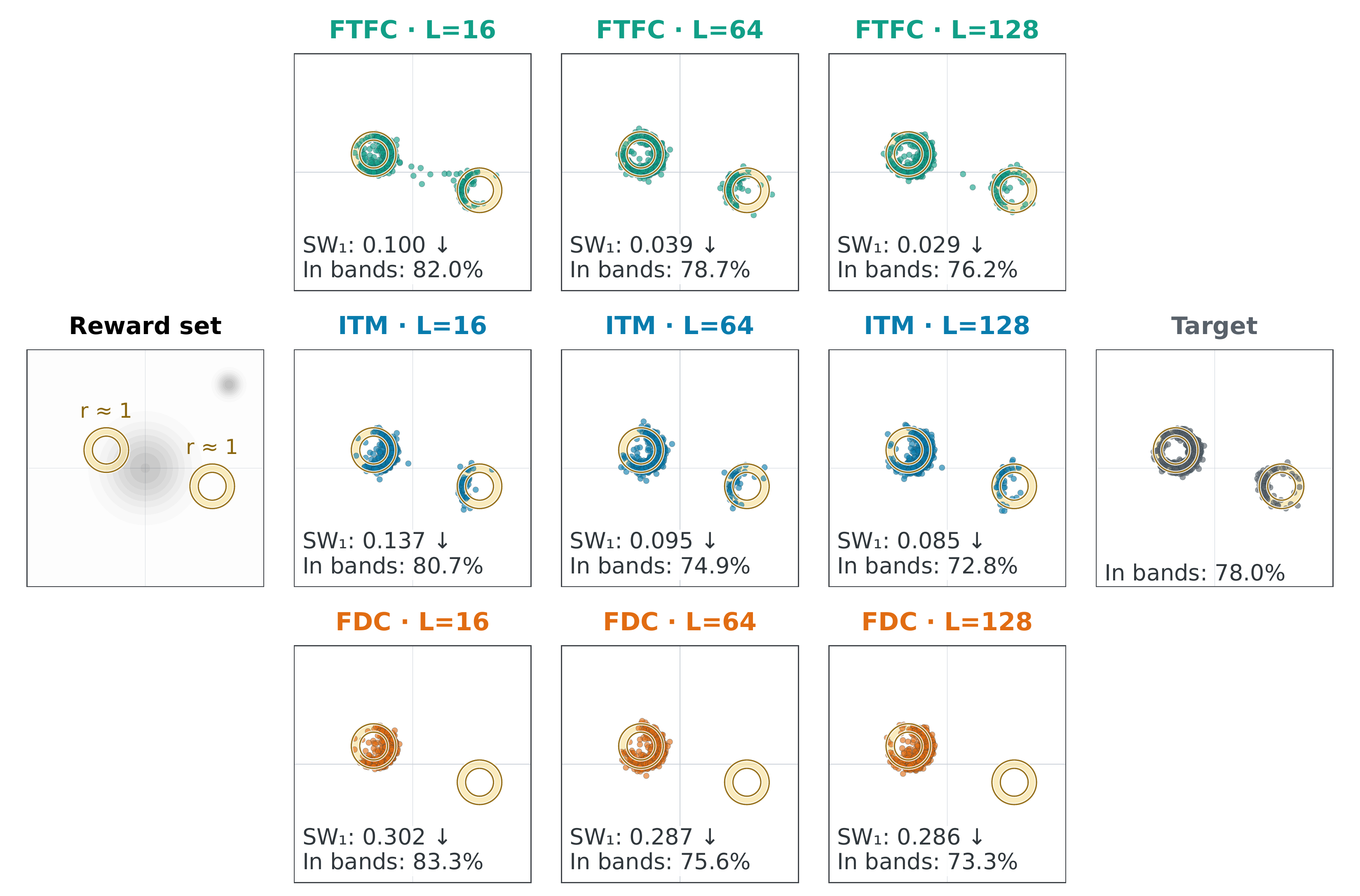}
  \caption{\textbf{Mean-reward objective}, $\mathcal F(p)=\mathbb E_p[r]$.
  The target is computed independently. Rows show FTFC, ITM, and FDC at
  $L=16,64,128$; the target is at middle right. Each panel shows 500 samples
  from one run; annotations use all 4,096 samples across three runs.}
  \label{fig:toy2d-mean}
\end{figure}

\paragraph{Convergence and cost.}
Figure~\ref{fig:toy2d-convergence} uses the same runs and numerical targets.
\begin{figure}[!ht]
  \centering
  \includegraphics[width=\linewidth]{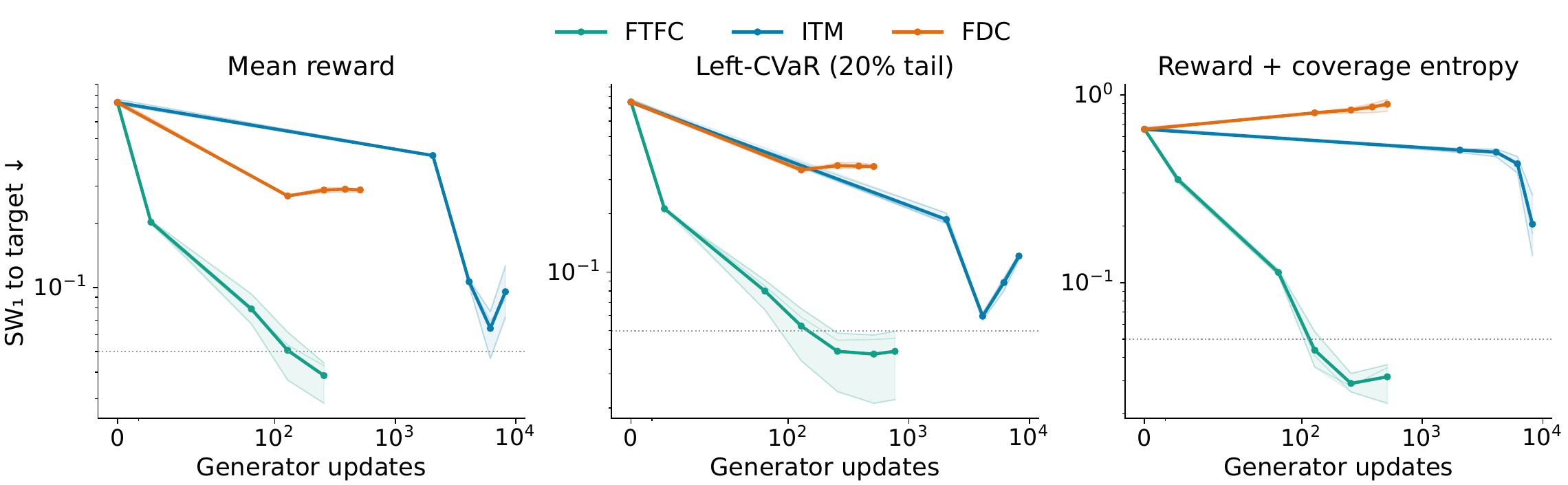}
  \caption{\textbf{Convergence to the numerical target at $L=64$.}
  Solid curves average three runs and shading spans their range. The dotted
  line marks sliced W1 $=0.05$. The axis counts generator updates; target
  optimization and endpoint acquisition are excluded. Batch sizes and sampling
  costs differ, so the curves do not establish equal-compute or end-to-end
  wall-clock speedups.}
  \label{fig:toy2d-convergence}
\end{figure}

\clearpage\subsection{Coverage entropy}
\label{app:toy2d-entropy}

The coverage objective $\mathcal F(p)=\mathbb E_p[r]+0.12H(\mathbb E_p[\phi])$
rewards entropy of aggregate feature masses, not the entropy of an individual
sample. It encourages both circles and their arcs (Figure~\ref{fig:toy2d-entropy}).
FTFC most closely matches this target; ITM concentrates on arcs and FDC largely
misses the right ring.

\begin{figure}[!ht]
  \centering
  \includegraphics[width=\linewidth]{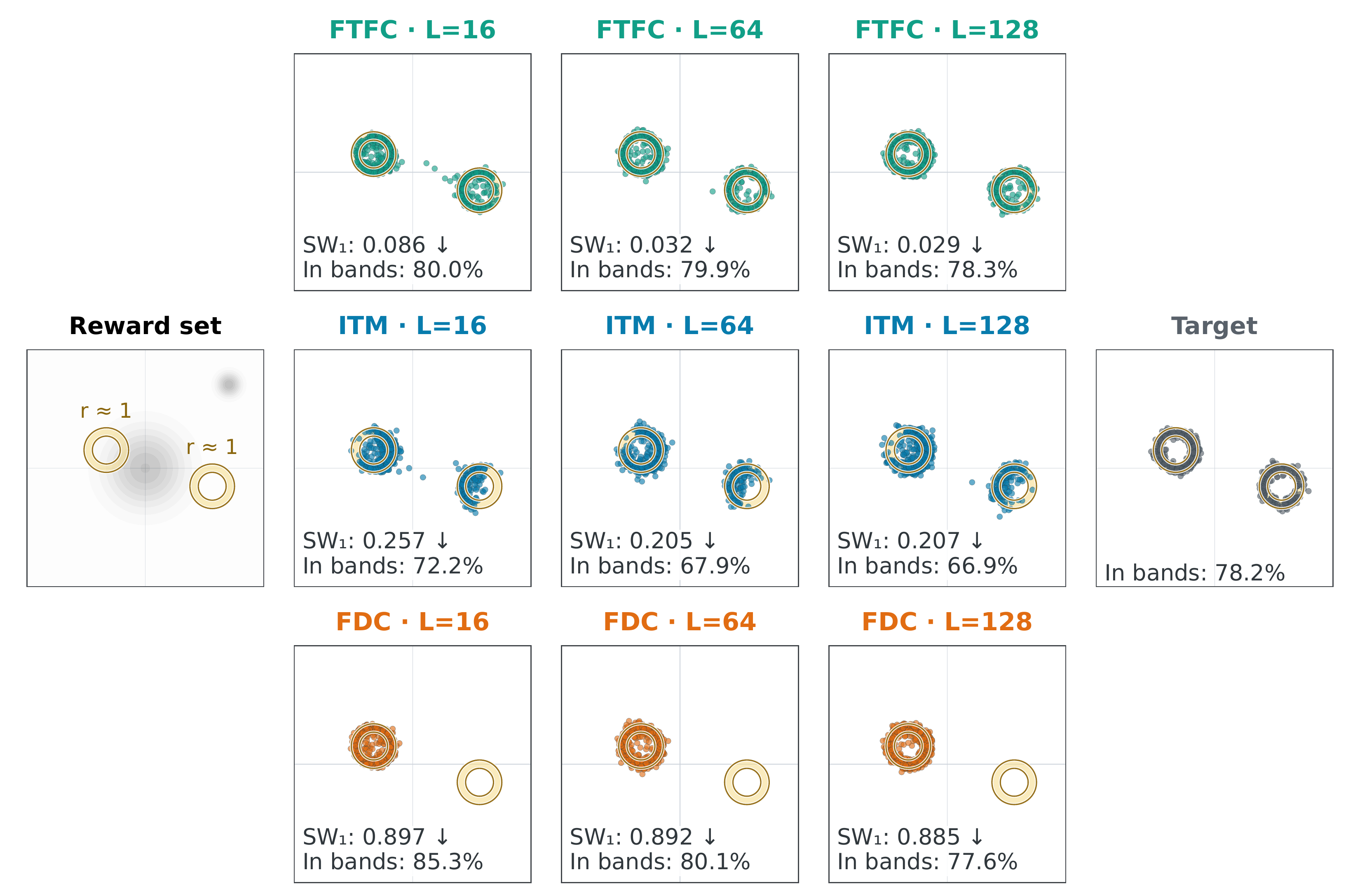}
  \caption{\textbf{Coverage-entropy objective.}
  $\mathcal F(p)=\mathbb E_p[r]+0.12H(\mathbb E_p[\phi])$ with KL coefficient
  $\alpha=0.03$. Rows, sampling lengths, evaluation protocol, and displayed
  run follow Figure~\ref{fig:toy2d-left-cvar}. The target retains mass outside
  the gold bands, so band-hit rate alone does not measure agreement.}
  \label{fig:toy2d-entropy}
\end{figure}

\end{document}